\documentclass[11pt]{article}
\usepackage[margin=1.1in,a4paper]{geometry}

\usepackage{mathtools}
\usepackage{amsmath, amssymb, bbm, natbib, graphicx, url, amsthm, subcaption, float, dsfont}
\usepackage{algorithm, algpseudocode}
\usepackage[mathcal]{eucal}
\usepackage{tikz}
\usepackage{verbatim}
 \usepackage{tcolorbox}
 \newtcolorbox{assbox}{colback=black!5!white,colframe=black!75!black}
  \newtcolorbox{thmbox}{colback=red!5!white,colframe=red!75!black}

\usepackage{amsfonts}       
\usepackage{enumitem}
\usepackage{microtype}      
\usepackage{xcolor}
\usepackage[colorlinks=true,citecolor=blue]{hyperref}    
\usepackage{diagbox}   
\usepackage{tcolorbox}
\usepackage{nicefrac}

\newcommand{\OT}{\mathrm{OT}}
\newcommand{\EOT}{\mathrm{EOT}}

\newcommand{\reg}{\mathrm{reg}}

\newcommand{\onl}{\mathrm{onl}}

\newcommand{\RR}{\mathbb{R}}
\newcommand{\NN}{\mathbb{N}}

\newcommand{\Cc}{\mathcal{C}}

\newcommand{\Nn}{\mathcal{N}}

\newcommand{\Uu}{\mathcal{U}}
\newcommand{\Vv}{\mathcal{V}}

\def\ones{\mathds{1}}

\DeclareMathOperator{\diag}{diag}

\newtheorem{theorem}{Theorem}[section]
\newtheorem{proposition}[theorem]{Proposition}
\newtheorem{lemma}[theorem]{Lemma}

\newtheorem{remark}[theorem]{Remark}

\DeclareMathOperator{\KL}{KL}

\DeclareMathOperator*{\argmin}{arg\,min}

\graphicspath{ {./images/} } 

\title{
Annealed Sinkhorn for Optimal Transport:\\ 
convergence, regularization path and debiasing
}

\author{
L\'ena\"ic Chizat\thanks{Ecole Polytechnique Fédérale de Lausanne (EPFL), Institute of Mathematics, 1015 Lausanne, Switzerland. \texttt{lenaic.chizat@epfl.ch}}
}

\begin{document}
\maketitle
\begin{abstract}
Sinkhorn's algorithm is a method of choice to solve large-scale optimal transport (OT) problems. In this context, it involves an inverse temperature parameter $\beta$ that determines the speed-accuracy trade-off. 
To improve this trade-off, practitioners often use a variant of this algorithm, \emph{Annealed Sinkhorn}, that uses an nondecreasing sequence $(\beta_t)_{t\in \NN}$ where $t$ is the iteration count.
However, besides for the schedule $\beta_t=\Theta(\log t)$ which is impractically slow, it is not known whether this variant is guaranteed to actually solve OT.
Our first contribution answers this question: we show that a concave annealing schedule asymptotically solves OT if and only if  $\beta_t\to+\infty$ and $\beta_t-\beta_{t-1}\to 0$. 
The proof is based on an equivalence with Online Mirror Descent and further suggests that the iterates of Annealed Sinkhorn follow the solutions of a sequence of relaxed, entropic OT problems, the \emph{regularization path}.
An analysis of this path reveals that, in addition to the well-known ``entropic'' error in $\Theta(\beta^{-1}_t)$, the annealing procedure induces a ``relaxation'' error in $\Theta(\beta_{t}-\beta_{t-1})$. The best error trade-off is achieved with the schedule $\beta_t = \Theta(\sqrt{t})$ which, albeit slow, is a universal limitation of this method. 
Going beyond this limitation, we propose a simple modification of Annealed Sinkhorn that reduces the relaxation error, and therefore enables faster annealing schedules. In toy experiments, we observe the effectiveness of our Debiased Annealed Sinkhorn's algorithm: a single run of this algorithm spans the whole speed-accuracy Pareto front of the standard Sinkhorn's algorithm.
\end{abstract}
 
\section{Introduction}
The optimal transport (OT) problem is a
linear program with a rich structure and a wide range of applications. In the discrete case, it involves two probability measures $p\in \Delta^*_n$ and $q\in \Delta_m^*$ where $\Delta^*_n \coloneqq \{ p\in (\RR_+^*)^n,\; p^\top \ones =1 \}$ is the tipless probability simplex, and a matrix $c\in \RR^{m\times n}$, where $c_{ij}$ represents the cost of exchanging a unit of mass between $p_i$ and $q_j$. The goal of the OT problem is to find a plan to transport the mass from $p$ to $q$ while minimizing the total transport cost. In mathematical terms, this is amounts to solving
\begin{align}\label{eq:OT}
\OT(p,q) \coloneqq \min_{\pi \in \Gamma(p,q)} \langle c,\pi\rangle,
\end{align}
where $\Gamma(p,q)\coloneqq\{\pi \in \RR^{m\times n}_+\;,\; \pi\ones = p \text{ and } \pi^\top \ones = q \}$ is the set of transport plans. In this paper, we focus on the discrete case for simplicity, but we will not rely fundamentally on the discrete structure and all our developments could be adapted to the continuous case.

An exact minimizer of the OT problem~\eqref{eq:OT} can be found in time $\tilde O((n+m)nm)$ via the network simplex algorithm or other specialized algorithms (see~\cite[Chap.~3]{peyre2019computational}). However, this can be prohibitively expensive for large scale problems where $m$ and $n$ are of the order of tens of thousands or more. This has motivated the development of algorithms to find approximate solutions with improved dependency in the size of the problem.

\paragraph{Sinkhorn's algorithm} One such algorithm, advocated for e.g.~in~\citep{kosowsky1991solving, cuturi2013sinkhorn} is Sinkhorn's algorithm, which solves the entropy-regularized optimal transport (EOT) problem
\begin{align}\label{eq:EOT}
\EOT_\beta(p,q)\coloneqq
 \min_{\pi \in \Gamma(p,q)} \langle c,\pi\rangle + \beta^{-1} \KL(\pi|\pi^{\mathrm{ref}})
\end{align}
where $\beta>0$ is the inverse temperature, $\KL$ is the Kullback-Leibler divergence and we take $\pi^{\mathrm{ref}}=(mn)^{-1}\ones_m \ones^\top_n$ by convention. It is not difficult to see that the unique minimizer of~\eqref{eq:EOT} is of the form $\pi^* = \diag(a^*)K\diag(b^*)$ for some $a^*\in (\RR_+^*)^m$, $b^*\in (\RR_+^*)^n$ and for $K=e^{-\beta c}\in \RR^{m\times n}_+$ (see Lem~\ref{lem:gibbs-plan-optimal}). Starting from $a_0=\ones$ and $b_0=\ones$, Sinkhorn's algorithm is equivalent to alternating $\KL$-projections of the ``primal'' iterate $\pi_t=\diag(a_t)K\diag(b_t)$ on the marginal constraints, which leads to the recursion
\begin{align}\label{eq:sinkhorn}
a_t = p \oslash (K b_{t-1}), &&
b_t = q \oslash (K^\top a_{t}).
\end{align}
These iterations have been rediscovered several times and bear various names in various communities (IPFP, RAS, matrix scaling, see the reviews~\citep{idel2016review, peyre2019computational}). When the final goal is to solve~\eqref{eq:OT}, the parameter $\beta$ determines the speed-accuracy trade-off: if chosen too large, it leads to a high regularization error; if chosen too small, it leads to a slow convergence of Sinkhorn's iterations (see a quantitative discussion in Section~\ref{sec:opt-quantitative}). While in this paper, we take as our end-goal to solve~\eqref{eq:OT}, let us mention that many works, starting from~\cite{cuturi2013sinkhorn}, have shown that solving~\eqref{eq:EOT} instead of~\eqref{eq:OT} is often beneficial in applications.

\paragraph{Annealed Sinkhorn} In the hope to improve the speed-accuracy trade-off, a classical heuristic, already put forth in \citep{kosowsky1991solving} and implemented in most computational OT packages\footnote{It is for instance implemented --  with a schedule of the form $\beta_t = \min \{\sigma^t, \beta_{\max} \}$ for some $\sigma>1$ -- in~\citep{flamary2021pot, feydy2019geomloss, cuturi2022optimal}, under the name ``$\epsilon$-scaling'' or ''$\epsilon$-scheduler''.}, consists in using a nondecreasing sequence $(\beta_t)_{t\in \NN}$ of inverse temperatures that depend on the iteration count $t$, a procedure known as (simulated) annealing.
This results in Alg.~\ref{alg:annealed-sinkhorn}, which we will refer to as \emph{Annealed Sinkhorn}.  The standard Sinkhorn's iterations~\eqref{eq:sinkhorn} correspond to the particular case where $(\beta_t)$ is constant. While the idea of annealing appears natural, it still lacks satisfying theoretical guarantees in this context. The only result that we are aware of, in~\citep{sharify2011solution}, shows that $(\pi_t)$ converges to an optimal transport plan for certain logarithmic annealing schedules $\beta_t = \Theta(\log t)$, which are impractically slow. Towards understanding this method, an analysis of the stability of dual solutions of~\eqref{eq:EOT} under changes of $\beta$ is carried out in~\cite[Sec.~3.2]{schmitzer2019stabilized}. This analysis gives indications on the number of Sinkhorn's iterations needed to steer the dual variables from $\beta$ to $\beta'>\beta$, but that work leaves open the question of convergence of the whole scheme. Relatedly, ~\cite{xie2020fast} studies the convergence of an inexact proximal point algorithm, which leads to iterates of the form of Alg.~\ref{alg:annealed-sinkhorn}. However, their theoretical approach requires to approximately solve~\eqref{eq:EOT} at each temperature, i.e.~$(\beta_t)$ must be piecewise constant, with no explicit bound on the time needed to be spent at each temperature.
Overall, the convergence analysis of Annealed Sinkhorn remains an open question. In~\cite[Chap.~3, p135]{feydy2020geometric}, it is presented as one of ``two major theoretical questions [...] still left to be answered [in the computational theory of entropic OT]''.

\begin{algorithm}
\caption{Annealed Sinkhorn (when $\beta_t\gg1$, should be implemented in variables $(u_t,v_t) = (\log(a_t),\log(b_t))$ with stabilized log-sum-exp operations to avoid numeric overflow).}\label{alg:annealed-sinkhorn}
\begin{enumerate}
\item \textbf{Input}: probability vectors $p\in \Delta_m^*$, $q\in \Delta_n^*$, cost $c\in \RR^{m\times n}$, annealing schedule $(\beta_t)_{t\geq 0}$
\item \textbf{Initialize}: let $b_0=\ones \in \RR^n$ and $K_0=e^{-\beta_0 c}\in \RR^{m\times n}$
\item \textbf{For $t=1,2,\dots$ let}
\begin{align*}
a_{t} &= p \oslash (K_{t-1} b_{t-1}), && \text{{\color{gray}$\#$ project on 1st marginal constraint} }\\
K_{t} &= e^{-\beta_{t} c},&& \text{{\color{gray}$\#$ update inverse temperature} }\\
b_{t} &= q\oslash (K_{t}^\top a_{t})&& \text{{\color{gray}$\#$ project on 2nd marginal constraint} }\\
\pi_t &= \diag(a_t)K_t\diag(b_t)&& \text{{\color{gray}$\#$ define primal iterate (can be done offline)} }\\
\end{align*}
\end{enumerate}
\end{algorithm}

\paragraph{Contributions} Towards addressing this open question, we make the following contributions:
\begin{itemize}
\item in Section~\ref{eq:consistency}, we characterize the family of positive concave annealing schedules $(\beta_t)_{t\in \NN}$ such that Annealed Sinkhorn asymptotically solves OT (Thm.~\ref{thm:AS-qualitative}). More specifically, we show that $(\pi_t)$ converges to an OT plan if and only if $\beta_t\to +\infty$ and $\beta_t-\beta_{t-1}\to 0$. Our proof is based on an online mirror descent (OMD) interpretation of Annealed Sinkhorn derived in Lem.~\ref{lem:OMD-sinkhorn}.
\item in Section~\ref{sec:path}, we study the \emph{regularization path} $(\pi^\mathrm{reg}_t)$, which is a tractable proxy for the iterates $(\pi_t)$ suggested by the OMD analysis. Thm.~\ref{thm:regularization-path} studies the convergence speed of $(\pi_t^\reg)$ towards OT solutions, and exhibits two types of errors: a well-known \emph{entropic} error, in $O(\beta_t^{-1})$ and a \emph{relaxation} error, in $\Theta(\beta_t-\beta_{t-1})$. This second term shows that to solve OT, it is necessary to slow down annealing. The best error trade-off is achieved with the schedule $\beta_t \propto t^{1/2}$.
\item Going beyond this limitation, we propose a simple modification, Debiased Annealed Sinkhorn (Alg.~\ref{alg:debiased-annealed-sinkhorn}), that reduces the relaxation error and therefore enables faster annealing schedules. It is justified by a technique to remove the first order relaxation error of the regularization path, presented in Prop.~\ref{prop:asymptotic-debiasing}. With a negligible extra cost, this technique is observed to systematically improve the behavior of Annealed Sinkhorn in our experiments.
\item Finally, we explain how to adapt our analysis to the Symmetric Sinkhorn algorithm in Section~\ref{sec:symmetric} where, interestingly, the regularization path solves a sequence of entropic \emph{unbalanced} OT problems.
\end{itemize}

\begin{figure}
\centering
\begin{subfigure}{0.5\linewidth}
\centering
\includegraphics[scale=0.47]{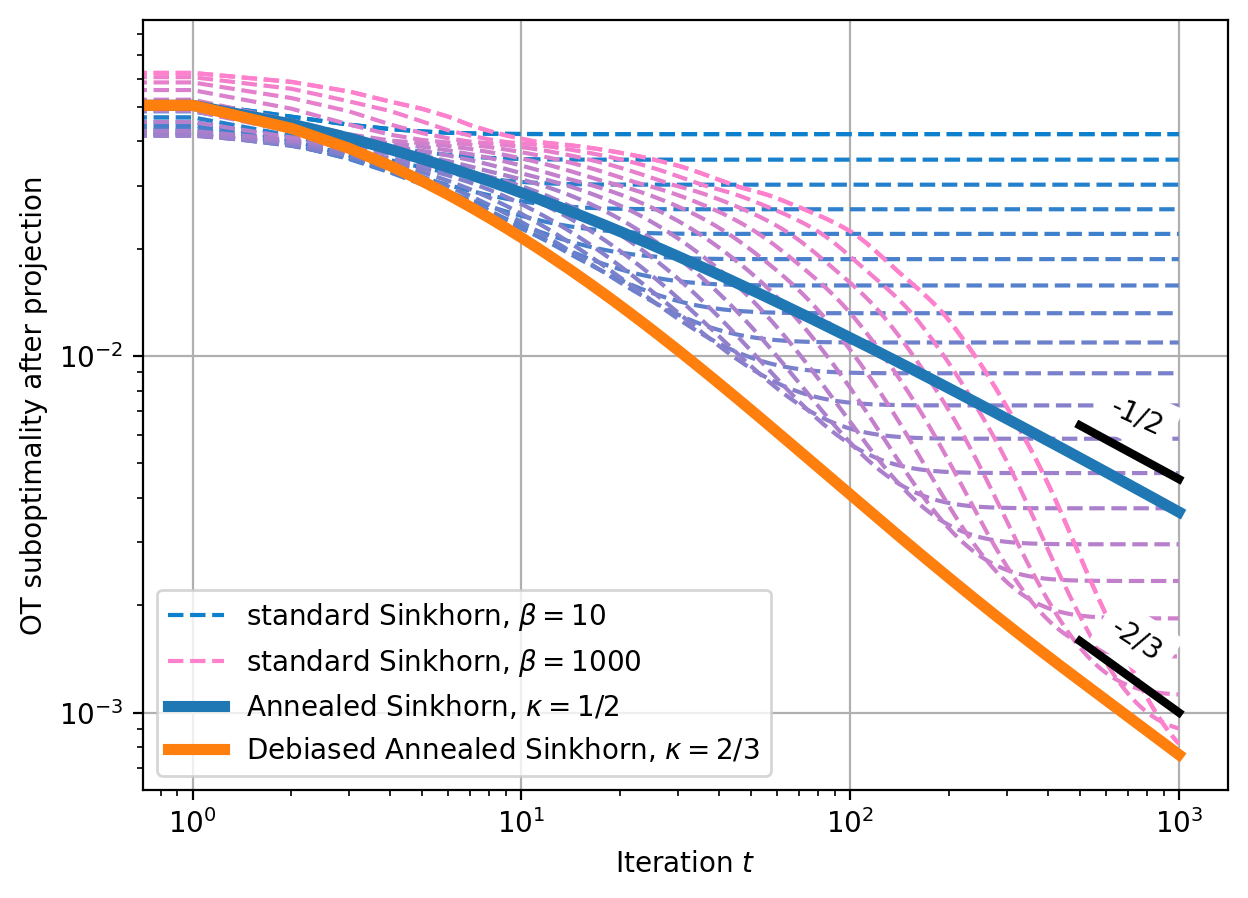}
\caption{Geometric cost ($c_{i,j}=\Vert x_i-y_j\Vert_2^2, \; x_i, y_j\in\RR^2$)}\label{fig:geometric}
\end{subfigure}%
\begin{subfigure}{0.5\linewidth}
\centering
\includegraphics[scale=0.47]{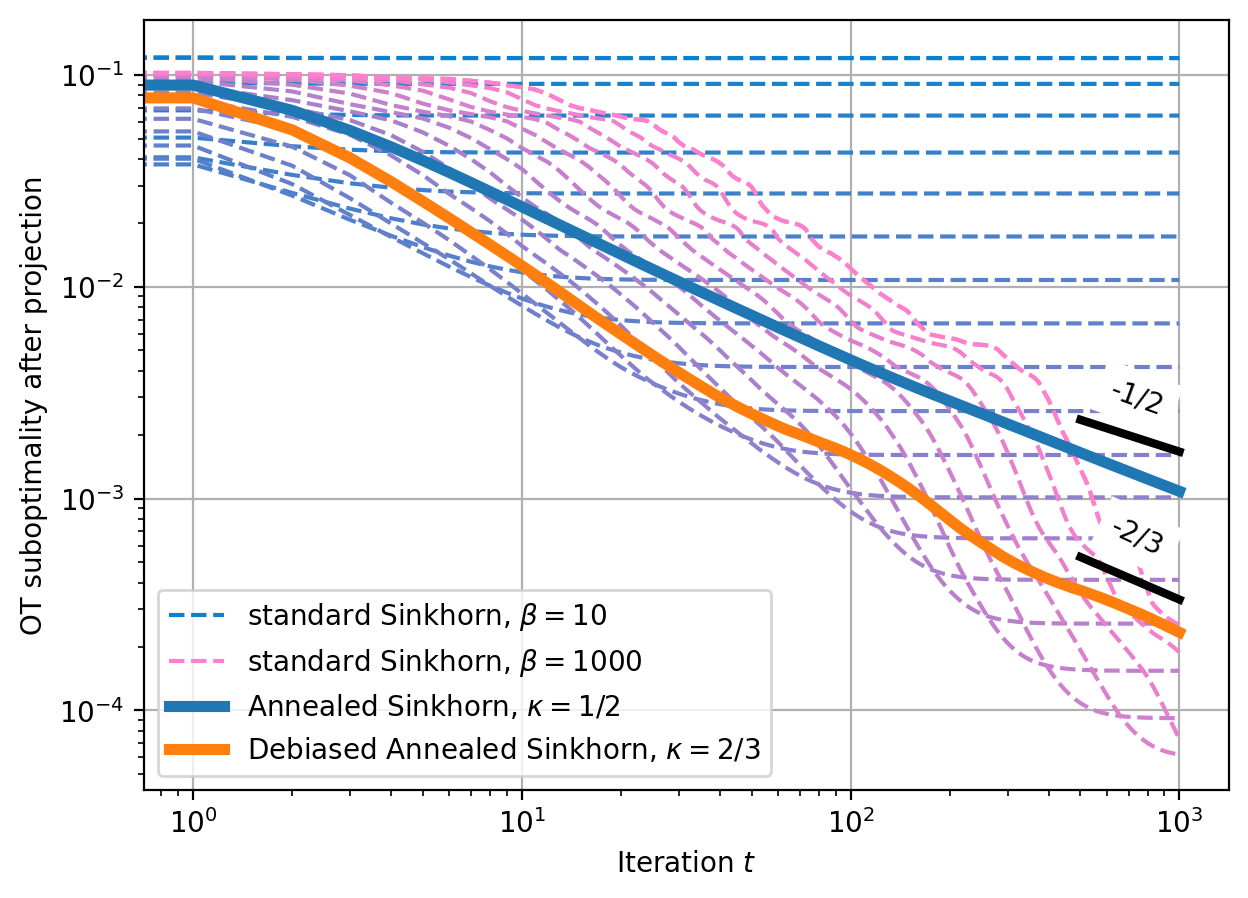}
\caption{Unstructured cost ($c_{i,j}\overset{\mathrm{iid}}{\sim} \Nn(0,1)$)}\label{fig:unstructured}
\end{subfigure}
\caption{Comparison of Sinkhorn's algorithm and its annealed variants for their respective optimal annealing schedules of the form $\beta_t=\beta_0 (1+t)^{\kappa}$ (here $\beta_0=(10/\Vert c\Vert_\mathrm{osc})$). We plot the OT suboptimality after projecting $\pi_t$ on $\Gamma(p,q)$ via Alg.~\ref{alg:projection}. The speed-accuracy Pareto front for Sinkhorn's algorithm is the pointwise minimum of the dashed lines. While Annealed Sinkhorn is far away from this front, the debiased version that we propose approaches or beats it.}\label{fig:pareto}
\end{figure}

The numerical experiments on Figure~\ref{fig:pareto} summarize our key insights. For Sinkhorn's algorithm and some of its variants discussed in this paper, we plot the OT suboptimality after ``projecting" the iterate $\pi_t$ on $\Gamma(p,q)$ via a cheap ``levelling'' procedure (Alg.~\ref{alg:projection}), i.e.~we plot $\langle c, \mathrm{proj}_{\Gamma(p,q)}(\pi_t)\rangle -\OT(p,q)$. The speed-accuracy Pareto front for Sinkhorn's algorithm is defined as the pointwise minimum of the dashed curves; each point in this front is achieved for a different value of $\beta$, and thus for a different run of the algorithm. Annealed Sinkhorn with its optimal schedule $\beta_t = \Theta(t^{1/2})$ is observed to converge at a rate $\Theta(t^{-1/2})$ (which is the rate we prove for the regularization path in Thm.~\ref{thm:regularization-path}) and often stays behind this Pareto front. In contrast, the proposed Debiased Annealed Sinkhorn can handle faster schedules -- here $\beta_t = \Theta(t^{2/3})$ -- and a single run of  this algorithm approaches or even beats the whole Pareto front (depending on the setting). We also observe that annealing is more beneficial in geometric contexts, where there is ``multiscale'' structure to exploit, but the theory presented in this present paper does not distinguish between various costs structures. Details on the experimental settings can be found in App.~\ref{app:experiments}. Let us mention that Sinkhorn's Pareto front can be improved by using acceleration methods (as, e.g.~in~\cite{thibault2021overrelaxed}). We do not discuss acceleration in the present paper but it is an interesting avenue for future research to study the interplay between acceleration and annealing.

\subsection{Notation}  We denote by $\ones_m$, or simply $\ones$ when the size is clear from the context, the vector of all ones in $\RR^m$. The Frobenius inner product between matrices or vectors is denoted by $\langle c,\pi\rangle=\sum_{i,j} c_{ij}\pi_{ij}$. For $a\in (\RR_+)^d$ and $b\in (\RR_+^*)^d$, the Kullback-Leibler divergence (a.k.a~relative entropy) is defined as $\KL(a|b)=\sum_i a_i\log(a_i/b_i) -a_i +b_i$ and it is jointly convex. We denote by $\odot$ and $\oslash$ the entry-wise product and division between vectors or between matrices of matching dimensions. All exponentials and logarithms act entry-wise on matrices and vectors. Recall that $\Gamma(p,q)$ denotes the set of transport plans between $p$ and $q$ (defined after Eq.~\eqref{eq:OT}) and we write $\Gamma(\ast,q)$ for the set of plans that only satisfy the second marginal constraint. The oscillation semi-norm of the cost matrix $c\in \RR^{m\times n}$ is defined as $\Vert c\Vert_\mathrm{osc} = \max_{i,j} c_{i,j} - \min_{i,j} c_{i,j}$. For a convex set $\Cc \subset \RR^d$, we denote by $\iota_\Cc$ the associated convex indicator function, defined as $\iota_{\Cc}(x)=0$ if $x\in \Cc$ and $+\infty$ otherwise.

\section{Convergence of Annealed Sinkhorn}\label{eq:consistency}

\subsection{The qualitative picture}
Our first result characterizes the asymptotic behavior of Annealed Sinkhorn (Alg.~\ref{alg:annealed-sinkhorn}) for \emph{concave} annealing schedules.

\begin{theorem}[Convergence of Annealed Sinkhorn]\label{thm:AS-qualitative}
Let $(\pi_t)$ be the sequence generated by Annealed Sinkhorn (Alg.~\ref{alg:annealed-sinkhorn}) with a positive, nondecreasing and \emph{concave} annealing schedule $(\beta_t)$, that is such that its difference sequence $\alpha_t = \beta_{t}-\beta_{t-1}$ in nonnegative and nonincreasing. Let $\lim \beta_t =\beta_\infty \in {]0,\infty]}$ and $\lim \alpha_t = \alpha_{\infty}\in {[0,+\infty[}$. Then any accumulation point $\pi_\infty$ of $(\pi_t)$ (at least one is guaranteed to exist) is a minimizer of
\begin{align}\label{eq:doubly-regularized}
\min_{\pi \in \Gamma(\ast,q)}
F_{\alpha_\infty,\beta_\infty}(\pi)&&\text{where}&& F_{\alpha, \beta}(\pi) = \langle c,\pi\rangle +\frac{1}{\alpha} \KL(\pi\ones|p) +\frac{1}{\beta} \KL(\pi|\pi^\mathrm{ref}).
\end{align}
In particular, $\pi_\infty$ is an optimal transport plan if and only if $\beta_\infty=+\infty$ and $\alpha_\infty=0$.
\end{theorem}
In this statement we use the conventions $\frac{1}{\infty} =0$ and  $\frac{1}{0} \KL(\tilde p|p)$ is the convex indicator of the equality constraint:
$$
\iota_{\{p\}}(\tilde p) = \begin{cases}
0 &\text{if $\tilde p=p$}\\
+\infty &\text{otherwise.}
\end{cases}
$$
Thus, any positive concave annealing schedule $(\beta_t)$ will fall in one of the three following cases:
\begin{itemize}
\item If $\lim \beta_t <+\infty$, then $\lim \alpha_t= 0$ and $\pi_\infty$ is the unique solution to $\EOT_{\beta_\infty}(p,q)$ (Eq.~\eqref{eq:EOT}).
\item If $\lim \beta_t =+\infty$ and $\lim \alpha_t =0$ then $\pi_\infty$ is a solution to $\OT(p,q)$ (Eq.~\eqref{eq:OT});
\item If $\lim \beta_t =+\infty$ and $\lim \alpha_t \in {]0,+\infty[}$, then $\pi_\infty$ is a solution to a \emph{relaxed} OT problem, which is one of the formulations of unbalanced OT, see~\cite{liero2018optimal}.
\end{itemize}
Let us give an intuitive explanation behind this behavior. By Lem.~\ref{lem:gibbs-plan-optimal}, the iterates $\pi_t$ are for any $t\in \NN^*$, the unique solution of $\EOT_{\beta_t}(\pi_t\ones,q)$ (defined in  Eq.~\eqref{eq:EOT}). Asymptotically solving $\OT$ thus requires progress on two fronts: (i) driving $\beta_t$ towards $+\infty$ so as to remove the entropic penalization, and (ii) driving $\pi_t \ones$ towards $p$.  Thm.~\ref{thm:AS-qualitative} shows a fondamental trade-off between these two goals: when $\beta_t$ grows too fast, the mechanism that drives $\pi_t \ones$ towards $p$ does not manage to catch-up. At the critical scaling $\beta_t =\Theta(t)$, annealed Sinkhorn implicitly solves a relaxed OT problem. A quantitative analysis of this trade-off is developed later  in Section~\ref{sec:path}.

\begin{proof}
First, the existence of accumulation points for $(\pi_t)$ is a consequence of the compactness of the simplex $\Delta_{m\times n}$. The rest of the proof relies on the online mirror descent guarantee of Lem.~\ref{lem:OMD-sinkhorn} presented in the next section. Observe that the iterates $(\pi_t)$ of the algorithm are invariant by adding a constant to the cost matrix. We can thus assume without loss of generality that $c\geq 0$ and use the guarantee Eq.~\eqref{eq:OMD-sinkhorn-pos} from Lem.~\ref{lem:OMD-sinkhorn}, with a different reference point $\pi^*$ depending on the case.

\textbf{The case $\alpha_\infty>0$.} We first consider the case $\alpha_\infty\in \RR_+^*$, which implies that $\beta_\infty=+\infty$. In this case, take $\pi^*$ be a minimizer of the problem in Eq.~\eqref{eq:doubly-regularized} with $(\alpha,\beta)=(\alpha_\infty,\infty)$. We have
$$
\alpha_t \langle c,\pi_t \rangle + \KL(\pi_t\ones|p) \leq  
\frac{\beta_t-\beta_0}{t}\langle c, \pi^*\rangle + \KL(\pi^*\ones|p)+ \frac{\KL(\pi^*|\pi_0)}{t} .
$$
Taking the limit $t\to\infty$ on both sides, using that $\frac{\beta_t-\beta_0}{t}\to \alpha_\infty$ and the lower-semicontinuity of $\KL$, we obtain
$$
 \alpha_\infty \langle c, \pi_\infty\rangle + \KL(\pi_\infty\ones|p) \leq 
 \alpha_\infty \langle c, \pi^*\rangle + \KL(\pi^*\ones|p)
$$
which proves the result in this case.

\textbf{The case $\alpha_\infty=0$.} In case $\alpha_\infty=0$, let $\pi^*\in \Gamma(p,q)$ be any transport plan. We have
$$
\alpha_t \langle c,\pi_t \rangle + \KL(\pi_t\ones|p) \leq  
\frac{\beta_t-\beta_0}{t}\langle c, \pi^*\rangle + \frac{\KL(\pi^*|\pi_0)}{t}.
$$
Again taking limits on both sides, using that $\frac{\beta_t-\beta_0}{t}\to 0$ and the lower-semicontinuity of $\KL$, we obtain that $\KL(\pi_\infty\ones |p)=0$, and hence $\pi_\infty \in \Gamma(p,q)$. It remains to show that $\pi_\infty$ is optimal for the problem $\EOT_{\beta_\infty}(p,q)$, where $\beta_\infty \in {]0,+\infty]}$. For this, we use the fact that by Lem.~\ref{lem:gibbs-plan-optimal}, $\pi_t$ is the unique minimizer for $\EOT_{\beta_t} (p_t,q)$ where $p_t\coloneqq \pi_t\ones\in \Delta_m$. It follows
\begin{multline*}
\langle c,\pi_t\rangle+\beta_\infty^{-1} \KL(\pi_t|\pi^{\mathrm{ref}}) = \EOT_{\beta_\infty} (p,q) 
+ \big\{\EOT_{\beta_t} (p_t,q)- \EOT_{\beta_\infty}(p,q)\big\} \\
+ \{(\beta_\infty^{-1}-\beta_t^{-1}) \KL(\pi_t|\pi^{\mathrm{ref}})\}.
\end{multline*}
Since $\KL(\pi_t|\pi^\mathrm{ref})$ is bounded and by Lem.~\ref{lem:gibbs-plan-optimal}, each of the bracketted terms converges to $0$ so, using the lower-semicontinuity of $\KL$, it follows  $\langle c,\pi_\infty \rangle+\beta_\infty^{-1} \KL(\pi_\infty|\pi^{\mathrm{ref}})\leq  \EOT_{\beta_\infty} (p,q)$ which is sufficient to conclude, since we already showed that $\pi_\infty\in \Gamma(p,q)$.
\end{proof}

\subsection{Equivalence with online mirror descent }
The starting point of our analysis of Annealed Sinkhorn is the observation that this algorithm can be interpreted as an online mirror descent; as detailed in the proof of the following lemma.

\begin{lemma}[OMD guarantee for Annealed Sinkhorn]\label{lem:OMD-sinkhorn}
Let $(\pi_t)$ be the sequence generated by Annealed Sinkhorn (Alg.~\ref{alg:annealed-sinkhorn}) with an annealing schedule $(\beta_t)\in (\RR_+^*)^{\NN}$ and consider the difference sequence $\alpha_t = \beta_{t}-\beta_{t-1}$. Then for any $t\geq 1$, it holds $F_{t}(\pi_t)\leq F_t(\pi_{t-1})$ and for any $\pi\in \Gamma(\ast,q)$,
\begin{align}\label{eq:OMD-sinkhorn}
\frac{1}{t}\sum_{k=1}^t \big(F_k(\pi_k)-F_k(\pi)\big)\leq \frac{\KL(\pi|\pi_0)}{t} && \text{where} && F_t(\pi) = \alpha_t \langle c,\pi\rangle + \KL(\pi\ones|p).
\end{align}
If moreover $c\geq 0$ and $(\alpha_t)$ is nonnegative and nonincreasing, it holds for any $\pi\in \Gamma(\ast,q)$,
\begin{align}\label{eq:OMD-sinkhorn-pos}
\alpha_t \langle c,\pi_t \rangle + \KL(\pi_t\ones|p) \leq  \KL(\pi\ones|p)+ \frac{\beta_t-\beta_0}{t}\langle c, \pi\rangle + \frac{\KL(\pi|\pi_0)}{t}.
\end{align}
\end{lemma}

Observe that for a constant $(\beta_t)$ -- that is, standard Sinkhorn's algorithm -- we have $(\alpha_t)=0$ and we recover from Eq.~\eqref{eq:OMD-sinkhorn-pos} applied to any $\pi \in \Gamma(p,q)$ the guarantee for Sinkhorn's iterates proved in~\cite[Prop.~8]{aubin2022mirror}:
\begin{align}\label{eq:sinkhorn-mirror}
\KL(\pi_t\ones|p)\leq \frac{ \inf_{\pi \in \Gamma(p,q)} \KL(\pi|\pi_0)}{t}.
\end{align}
which we discuss again later in Rmk.~\ref{rmk:rates-sinkhorn} (note that in Eq.~\eqref{eq:sinkhorn-mirror}, we can also take the infimum over $a_0$ in the expression of $\pi_0$ since $a_0$ is never used by the algorithm).
\begin{proof}
This result relies on an \emph{online mirror descent} (OMD) interpretation of Alg.~\ref{alg:annealed-sinkhorn} (we give general background on OMD in App.~\ref{sec:MD}). More specifically, let us show that $(\pi_t)$ coincides with the iterates $(\tilde \pi_t)$ of OMD with step-size $1$ in the $\KL$ geometry starting from $\pi_0$, on the sequence of convex objective functions
$$
F_t(\pi) = \alpha_t\langle c,\pi\rangle +\KL(\pi\ones|p)+\iota_{\{q\}}(\pi^\top \ones)
$$
where $\iota_C(x)=0$ if $x\in C$ and $+\infty$ otherwise is the convex indicator function. The base case $\tilde \pi_0=\pi_0$ is true by construction, so let us assume $\tilde \pi_{t-1}=\pi_{t-1}=\diag(a_{t-1})K_{t-1}\diag(b_{t-1})$ and prove that this implies $\pi_{t}=\tilde \pi_{t}$. One OMD step is defined (see Alg.~\ref{alg:onlineMD}) by
\begin{align*}
\tilde \pi_{t} &\coloneqq \argmin_{\pi\in \RR_+^{m\times n}} \alpha_t \langle c,\pi\rangle + \langle \log((\pi_{t-1}\ones) \oslash p)\ones^\top , \pi\rangle + \iota_{\{q\}}(\pi^\top \ones) + \KL(\pi|\pi_{t-1}).
\end{align*}
Introducing Lagrangian variables for the constraints leads to the saddle problem
$$
\min_{\pi\in \RR_+^{m\times n}} \max_{v\in \RR^n}\; \alpha_t \langle c,\pi\rangle + \langle \log((\pi_{t-1}\ones) \oslash p)\ones^\top, \pi\rangle +\langle v,  \pi^\top \ones-q\rangle+ \KL(\pi|\pi_{t-1})
$$
and the minimizer $\tilde \pi_{t}$ is characterized by the KKT conditions: $\tilde \pi_t^\top \ones =q$ and $\exists v\in \RR^n$ such that $G(\tilde \pi_{t},v)=0$ where 
$$
G(\pi,v) \coloneqq \alpha_t c + \log((\pi_{t-1}\ones) \oslash p) \ones^\top + \ones v^\top+\log(\pi\oslash \pi_t) \in \RR^{m\times n}.
$$
In order to prove that $\pi_{t}=\tilde \pi_{t}$, it is thus sufficient to prove that there exists $v\in \RR^n$ such that $G(\pi_{t},v) = 0$, because the first condition $\pi_{t}^\top \ones=q$ is readily guaranteed by the fact that $b_t=q\oslash (K_t^\top a_t)$. Since $\pi_{t}=\diag(a_{t})K_{t}\diag(b_{t})$, we have
\begin{multline*}
G(\pi_{t},v) = \alpha_t c + \log((a_{t-1}\odot (K_{t-1}b_{t-1})) \oslash p) \ones^\top +\ones v^\top\\
+\log(K_{t}\oslash K_{t-1})+\log(a_{t}\oslash a_{t-1})\ones^\top +\ones \log(b_{t}\oslash b_{t-1})^\top.
\end{multline*}
A sufficient condition for $M(\pi_{t},v) =0$ is thus
\begin{multline*}
\left\{
\begin{aligned}
\alpha_t c+\log(K_{t}\oslash K_{t-1})=0\\
 \log((a_{t-1}\odot (K_{t-1}b_{t-1})) \oslash p)+ \log(a_{t}\oslash a_{t-1})=0\\
v+\log(b_{t}\oslash b_{t-1})=0
\end{aligned}
\right.
\Leftrightarrow
\left\{
\begin{aligned}
K_{t}=K_{t-1}\odot e^{-\alpha_tc}\\
a_{t} = p\oslash (K_{t-1}b_{t-1})\\
v=-\log(b_{t}\oslash b_{t-1})
\end{aligned}
\right..
\end{multline*}
By taking $v$ as defined by the last equation, and recalling how Alg.~\ref{alg:annealed-sinkhorn} defines $K_t$ and $a_t$, it clearly holds $M(\pi_{t},v)=0$. This shows that  $\pi_t=\tilde \pi_t$, and thus, by recursion,  $(\pi_t)=(\tilde \pi_t)$.

Let us now derive the optimization guarantee. Clearly each $F_t$ is convex and by Lem.~\ref{lem:relative-smoothness}, the smooth part of the objective is $1$-relatively smooth, so OMD theory applies with our choice of step-size $1$ and we obtain Eq.~\eqref{eq:OMD-sinkhorn} by an application of Prop.~\ref{prop:OMD}.

Finally, if we additionally assume that  $(\alpha_t)$ is nonnegative and nonincreasing and $c\geq 0$, we have $\alpha_t\langle c,\pi\rangle\leq \alpha_{t-1}\langle c,\pi\rangle$ for any $\pi\in \RR^{m\times n}_+$. Combining this with the OMD guarantee that $F_t(\pi_t)\leq F_t(\pi_{t-1})$, it follows that
$$
F_t(\pi_t)\leq F_t(\pi_{t-1})\leq F_{t-1}(\pi_{t-1}),
$$
i.e., the sequence $(F_{t}(\pi_t))_{t\in \NN}$ is nonincreasing. This allows  us to deduce from Eq.~\eqref{eq:OMD-sinkhorn} that
\begin{align}
F_t(\pi_t)- \frac1t\sum_{k=1}^t F_t(\pi) \leq \frac{\KL(\pi|\pi_0)}{t}
\end{align}
and Eq.~\eqref{eq:OMD-sinkhorn-pos} follows by rearranging the terms and using the fact that $\sum_{k=1}^t \alpha_k = \beta_t-\beta_0$.
\end{proof}

\subsection{A discussion on quantitative guarantees}\label{sec:opt-quantitative}

\paragraph{Quantifying progress towards OT}
Since our purpose is to solve OT but the iterates $(\pi_t)$ are not in the constraint set $\Gamma(p,q)$, it is not obvious a priori how to measure progress towards solving OT. A solution to this problem was proposed in~\cite{altschuler2017near}: given any $\pi\in \Delta_{m\times n}$, they propose a computationally cheap ``projection'' procedure, shown in Alg.~\ref{alg:projection} that takes any plan and builds an $\ell_1$-approximation of it that belongs to $\Gamma(p,q)$. From the proof of~\cite[Thm.~1]{altschuler2017near}, we extract the following lemma.
\begin{lemma}[\cite{altschuler2017near}]\label{lem:altschuler}
If $\pi \in \Delta_{m\times n}$ is of the form $\pi = \diag(a) e^{-\beta c}\diag(b)$ for some $\beta>0$, $a\in \RR^m_+$ and $b\in \RR^n_+$, then
$$
\langle c,\mathrm{proj}_{\Gamma(p,q)}(\pi)\rangle -\OT(p,q) \leq \beta^{-1}\log(mn) + 4(\Vert \pi\ones -p\Vert_1+\Vert \pi^\top \ones -q\Vert_1)\Vert c\Vert_\mathrm{osc}.
$$
where $\mathrm{proj}_{\Gamma(p,q)}(\pi)\in \Gamma(p,q)$ is the output of Alg.~\ref{alg:projection} applied to $\pi$ (which takes $O(nm)$ time to compute).
\end{lemma}

This lemma allows to quantify, on fair grounds, the progress made towards solving OT for any algorithm with iterates of the form $\diag(a_t)e^{-\beta_t c}\diag(b_t)$. The first error term is due to entropic regularization, and the second error term is due to the distance from the constraint set $\Gamma(p,q)$. 

\begin{algorithm}
\caption{``Projection'' on transport plans~\citep{altschuler2017near}}\label{alg:projection}
\begin{enumerate}
\item \textbf{Input}: probability vectors $p\in \Delta_m$, $q\in \Delta_n$, and $\pi\in \RR_+^{m\times n}$.
\begin{enumerate}
\item $\pi'=\diag(a)\pi$ where $a = \min\{\ones, p \oslash (\pi \ones)\}$
\item $\pi''=\pi'\diag(b)$ where $b =  \min\{\ones, q \oslash ((\pi')^\top \ones)\}$ 
\item $\Delta_p = p - (\pi'')\ones $ and $\Delta_q =q -(\pi'')^\top \ones$
\end{enumerate}
\item \textbf{Return:} $\bar \pi = \pi''+\Delta_p\Delta_q^\top/\Vert \Delta_p\Vert_1$ 
\end{enumerate}
\end{algorithm}

\paragraph{Complexity of OT via standard Sinkhorn} Following \cite{altschuler2017near}, one can then use standard Sinkhorn's algorithm to build a transport plan with sub-optimality $\epsilon>0$ as follows: pick the inverse temperature $\beta =2/(\epsilon\log (mn))$ and run Sinkhorn's iterations until $\Vert\pi_t\ones -p\Vert_1\leq \epsilon/(4\Vert c\Vert_\infty)$. Then Lem.~\ref{lem:altschuler} guarantees that $\mathrm{proj}_{\Gamma(p,q)}(\pi_t)$ solves the problem. ~\cite{dvurechensky2018computational} provides, for Sinkhorn's iterates, the bound  $\Vert\pi_t\ones -p\Vert_1 = \tilde O(\beta \Vert c\Vert_\infty/t)$ (where $\tilde O()$ hides some logarithmic terms), so the stopping condition can be achieved in $\tilde O(\beta \Vert c\Vert_\infty^2/\epsilon)$ iterations. Counting $O(mn)$ elementary operations per iteration, this leads to an overall complexity $\tilde O(mn\Vert c\Vert_\infty^2/\epsilon^2)$ to output an $\epsilon$-optimal transport plan.

\paragraph{Complexity of OT via Annealed Sinkhorn}
One may wonder if one can obtain comparable complexity bounds by using Annealed Sinkhorn (Alg.~\ref{alg:annealed-sinkhorn}) and our optimization guarantees (Lem.~\ref{lem:OMD-sinkhorn}). Unfortunately, we failed to obtain complexity guarantees matching those of Sinkhorn's algorithm discussed in the previous paragraph. In the following, we explain where the weakness appears in our analysis.

By Eq.~\eqref{eq:OMD-sinkhorn-pos} evaluated at $\pi^*$ an OT plan we have, for a positive, nondecreasing and concave schedule $(\beta_t)$,
\begin{align}
\KL(\pi_t\ones|p) &\leq  \frac{\beta_t-\beta_0}{t}\langle c, \pi^*\rangle - \alpha_t \langle c,\pi_t \rangle + \frac{\KL(\pi^*|\pi_0)}{t}\\
&\leq \alpha_t \langle c, \pi^*-\pi_t\rangle + \Big(\frac{\beta_t-\beta_0}{t}-\alpha_t \Big)\langle c, \pi^*\rangle +\frac{\KL(\pi^*|\pi_0)}{t}.
\end{align}
By Pinsker's inequality (Lem.~\ref{lem:pinsker} below), and the basic upper bound $\sqrt{2a+2b+2d}\leq \sqrt{6\max\{a,b,d\}}\leq 3\sqrt{a}+3\sqrt{b}+3\sqrt{d}$ for $a,b,d\geq 0$, it follows
\begin{align}\label{eq:error-decomp-discussion}
\Vert \pi_t \ones -p \Vert_1 \leq 3 \sqrt{\alpha_t \vert \langle c, \pi^*-\pi_t\rangle\vert} + 3\sqrt{\Big\vert\frac{\beta_t-\beta_0}{t}-\alpha_t \Big\vert \vert\langle c, \pi^*\rangle\vert} + 3\sqrt{\frac{\KL(\pi^*|\pi_0)}{t}}.
\end{align}
 To understand how these terms behave, let us consider polynomial annealing schedules of the form $\beta_t=(1+t)^\kappa$ for $\kappa\in {]0,1[}$, for which we have $\alpha_t = \beta_t-\beta_{t-1} = \Theta(t^{\kappa-1})$ and $\frac{\beta_t-\beta_{0}}{t}-\alpha_t = \Theta(t^{\kappa-1})$. Focusing exclusively on the  exponents on $t$, it follows from Eq.~\eqref{eq:error-decomp-discussion} that
$
\Vert \pi_t \ones -p \Vert_1 = O(t^{\frac{\kappa-1}{2}}+t^{-1/2})$. With the help of Lem.~\ref{lem:altschuler}, it then follows
$$
\langle c,\mathrm{proj}_{\Gamma(p,q)}(\pi_t)\rangle -\OT(p,q) = O\Big(t^{-\kappa} +t^{\frac{\kappa-1}{2}}+t^{-1/2}\Big).
$$
This bound achieves its fastest decay for the choice $\kappa=\frac13$ which leads to an overall complexity scaling in $O(\epsilon^{-3})$ to build $\epsilon$-optimal transport plan, which is worse than the best known rate via Sinkhorn's algorithm which is in $O(\epsilon^{-2})$, as discussed above.

The main weakness comes from the second term in Eq.~\eqref{eq:error-decomp-discussion}, a consequence of the fact that OMD only gives guarantees of average type (while a naive bound on the first term in Eq.~\eqref{eq:error-decomp-discussion} also gives $O(t^{\frac{\kappa-1}{2}})$, a finer analysis using the fact that $\pi_t\to \pi^*$ might improve this exponent). Since this weakness is not observed in numerical experiments, this motivates our switch of focus, from the next section onwards, to the \emph{regularization path} which better reveals the behavior of the algorithm than the OMD bound. 

\begin{lemma}[Pinsker's inequality] \label{lem:pinsker}
For any $p,p'\in \Delta^{m}$, it holds $\Vert p -p' \Vert_1 \leq \sqrt{2\KL(p|p')}$.
\end{lemma}

\begin{remark}[Optimization bounds galore for Sinkhorn]\label{rmk:rates-sinkhorn} Another source of weakness in our optimization guarantees is that our starting point, Eq.~\eqref{eq:OMD-sinkhorn}, has a suboptimal rate already in the particular case of standard Sinkhorn. Indeed, Eq.~\eqref{eq:sinkhorn-mirror} gives a convergence rate $\Vert \pi_t\ones -p\Vert_1 = O(\sqrt{\beta \Vert c\Vert_\infty/t})$ via Pinsker's inequality, while faster rates in $\Vert \pi_t\ones -p\Vert_1=\tilde O(\beta \Vert c\Vert_\infty/t)$ are known to hold~\citep{dvurechensky2018computational, ghosal2022convergence}. Let us also mention that another class of guarantees give exponential convergence rates, but generally with a poor dependency in $\beta$, such as  $\Vert \pi_t\ones -p\Vert_1 = O(e^{-t e^{-\beta \Vert c\Vert_{\mathrm{osc}}}})$~\citep{franklin1989scaling}. The latter guarantee is the starting point of the analysis of Annealed Sinkhorn in~\cite{sharify2011solution}; and its poor dependency in $\beta$ is the reason why they can only cover extremely slow schedules $\beta_t=\Theta(\log(t))$. Recently, much sharper exponential rates in $e^{-\Theta(t /(\beta \Vert c\Vert_{\mathrm{osc}})^{\kappa})}$ were obtained in~\cite{chizat2024sharper}, in ``geometric settings'', that is under regularity assumptions on the cost and on one of the  marginals, and where the exponent $\kappa\in \{1,2\}$ depends on the precise assumptions. In the spirit of~\cite{sharify2011solution}, one could use these faster rates as the starting point of an analysis of Annealed Sinkhorn but now with polynomial schedules. 
We do not develop this approach here because the focus of the present paper is on universal phenomena which are independent of the structure of the OT problem.
\end{remark}

\section{The regularization path}\label{sec:path}

\subsection{Proxies for the optimization path}
In order to improve our understanding of Annealed Sinkhorn, we now turn our interest towards the analysis of theoretically motivated \emph{proxies} for the optimization path $(\pi_t)_{t\in \NN}$.

\paragraph{The Online Path} As seen in the proof of Lem.~\ref{lem:OMD-sinkhorn}, at  each times step $t$, $\pi_t$ is obtained by taking a mirror descent step on the function $F_t(\pi)=\alpha_t \langle c,\pi\rangle +\KL(\pi\ones |p)+\iota_{\{q\}}(\pi^\top \ones)$, and is of the form $\pi_t=\diag(a_t)e^{-\beta_t c}\diag(b_t)$. We define the \emph{online path} as the result of fully optimizing $F_t$ under this constraint set: 
\begin{align}\label{eq:online-path}
\min_{\pi \in \Cc_{\beta_t} \cap \Gamma(\ast,q)} \alpha_t\langle c, \pi\rangle + \KL(\pi\ones|p) &&\!\text{where}\! && \Cc_\beta =\Big\{ \diag(a)e^{-\beta c}\diag(b), a\in \RR_+^m, b\in \RR_+^n\Big\}.
\end{align}
The objective is lower semicontinuous on the compact set $\Cc_\beta\cap \Gamma(\ast,q)$, so a minimizer is guaranteed to exist.  We denote by $(\pi^\text{onl}_t)_{t\in \NN}$ a path of (a selection of) minimizers associated to a given annealing schedule $(\beta_t)_{t\in \NN}$ (remember that the schedule also determines $(\alpha_t)_{t\in \NN^*}$ via $\alpha_t = \beta_t-\beta_{t-1}$). Unfortunately, the optimization problem in Eq.~\eqref{eq:online-path} does not seem to yield convenient optimality conditions. This motivates us to study another proxy for the optimization path.

\paragraph{The Regularization Path}  Consider the sequence of optimization problems
\begin{align}\label{eq:Fprox}
\min_{\pi\in \Gamma(\ast,q)} F^{\mathrm{reg}}_t(\pi)&& \text{where} && F^{\mathrm{reg}}_t(\pi) = \langle c, \pi\rangle +\frac{1}{\alpha_t} \KL(\pi\ones|p)+\frac{1}{\beta_t}\KL(\pi|\pi^{\mathrm{ref}})
\end{align}
and by convention\footnote{This choice of $\pi^{\mathrm{ref}}$ induces an entropic penalty on $\pi^{\mathrm{reg}}\ones$ (see e.g.~\cite[Lem.~1.6]{marino2020optimal}) which is absent from Eq.~\eqref{eq:online-path}.  It is not excluded that other choices such as $\pi^{\mathrm{ref}} = (\pi\ones)\otimes q$  define better proxies for $(\pi_t)$, but our analysis is not precise enough to distinguish between these choices.} $\pi^{\mathrm{ref}}=(mn)^{-1}\ones_m\ones_n^\top$.
Each of these problems is a strictly convex minimization problem with a unique minimizer. We denote by $(\pi^{\mathrm{reg}}_t)_{t\in \NN^*}$ the resulting sequence of solutions, and call it the \emph{regularization path}. We motivate our definition of the regularization path in two ways. First, we show experimentally on Fig.~\ref{fig:optim-vs-proxy} that it tracks the trajectory of Annealed Sinkhorn to a high accuracy: the distance between these two paths appear to be an order of magnitude below their distance from OT solutions. Second, the regularization path can be related to the online path thanks to the following estimate which compares their first marginals. Since both paths are solutions to $\EOT_{\beta_t}(\ast,q)$ problems (by Lem.~\ref{lem:gibbs-plan-optimal}), the difference between these two paths is entirely captured by the distance between their first marginals. Let us mention that the bound from Prop.~\ref{prop:onl-vs-reg} is pessimistic compared to what is seen on Fig.~\ref{fig:optim-vs-proxy}; this means that either it is loose, or that $(\pi_t^\reg)$ approximates $(\pi_t)$ in fact \emph{better} than $(\pi_t^\onl)$.

\begin{proposition}[Online vs.~Regularization paths]\label{prop:onl-vs-reg} For any $t\in \NN^*$, it holds
$$
\Vert \pi^{\onl}_t\ones-\pi_t^\reg\ones\Vert_1\leq \sqrt{\frac{2\alpha_t\log(mn)}{\beta_t}}.
$$
In particular, for polynomial schedules $\beta_t=\beta_0(1+t)^\kappa$, $\kappa>0$, this bound is in $O(t^{-1/2})$.
\end{proposition}
\begin{proof}
Since for any $\pi\in \Delta_{m\times n}$
we have $0\leq \KL(\pi|\pi^{\mathrm{ref}})\leq \log(mn)$, it follows
\begin{align*}
F_t^{\reg}(\pi^\onl_t)- F_t^\reg(\pi^\reg_t) 
& \leq \Big\{\langle c, \pi^{\onl}_t-\pi^{\reg}_t\rangle + \frac{1}{\alpha_t}\Big(\KL(\pi^{\onl}_t\ones|p)-\KL(\pi^{\reg}_t\ones|p)\Big)\Big\}+\frac{1}{\beta_t} \log(mn).
\end{align*}
By optimality of $\pi^{\onl}_t$ in Eq.~\eqref{eq:online-path}, and since $\pi^{\reg}_t\in \Cc_{\beta_t}$, the quantity in curly brackets is nonpositive. It follows $F^{\reg}_t(\pi^{\onl}_t)- F^{\reg}_t(\pi^{\reg}_t) \leq \frac{1}{\beta_t} \log(mn)$. Now, for $p'\in \Delta^m$, consider the function 
$$
G(p') \coloneqq \min_{\pi \in \Gamma(p',q)} \alpha_t F_t^{\reg}(\pi) = \KL(p'|p) +\Big\{\min_{\pi \in \Gamma(p',q)} \alpha_t \langle c, \pi\rangle +\frac{\alpha_t}{\beta_t}\KL(\pi|\pi^{\mathrm{ref}})\Big\}.
$$
This function satisfies $G(\pi\ones) \leq \alpha_t F^{\reg}_t(\pi)$ and for any $\pi\in \Gamma(\ast,q)$ with equality for $\pi^{\reg}_t$ and $\pi^{\reg}_t\ones$ is the unique global minimizer of $G$. The quantity in brackets is equal to  $\alpha_t \EOT_{\beta_t}(p',q)$, which is a convex function of $p'$ (this can be seen from the dual formulation in Lem.~\ref{lem:gibbs-plan-optimal}: it is a supremum of linear forms of $p'$, hence convex). Since $\KL(\cdot|p)$ is $1$-relatively strongly convex with respect to itself, we have, for any $p'\in \Delta_m$, that $
\KL(p',\pi^{\reg}_t\ones) \leq G(p')-G(\pi^{\reg}_t\ones).
$
Applying this strong-convexity estimate to $\pi^{\onl}_t\ones$, it follows 
$$
\KL(\pi^{\onl}_t\ones|\pi^{\reg}_t\ones) \leq G(\pi^{\onl}_t\ones)-G(\pi^{\reg}_t\ones) \leq \alpha_t \big(F^{\reg}_t(\pi_t^{\onl})-F^{\reg}_t(\pi_t^{\reg}) \big)\leq \frac{\alpha_t}{\beta_t}\log(mn).
$$
The estimate of Prop.~\ref{prop:onl-vs-reg} then follows by Pinsker's inequality (Lem.~\ref{lem:pinsker}). For a polynomial annealing schedule $\beta_t = \beta_0(1+t)^\kappa=\Theta(t^{\kappa})$ with $\kappa>0$, we have $\alpha_t=\Theta(t^{\kappa-1})$ hence $\sqrt{\alpha_t/\beta_t} = O(t^{-1/2})$.
\end{proof}

\subsection{Convergence rate of the regularization path}
Now that we have motivated the regularization path as a relevant proxy for the optimization path, let us quantify the evolution of its sub-optimality. To do this properly, we adopt the methodology presented in Section~\ref{sec:opt-quantitative}, and bound the sub-optimality after ``projection'' on the set of transport plans. 

\begin{theorem}[Convergence rate of the regularization path]\label{thm:regularization-path}
Let $(\pi^\mathrm{reg}_t )$ be the regularization path associated to an annealing schedule $(\beta_t)$, and let $\bar \pi_t \coloneqq \mathrm{proj}_{\Gamma(p,q)}(\pi^\mathrm{reg}_t)$ be the output of the projection algorithm Alg.~\ref{alg:projection} applied to $\pi_t^\mathrm{reg}$, which satisfies $\bar \pi_t\in \Gamma(p,q)$. It holds
\begin{align}\label{eq:reg-path-bound}
\langle c,\bar \pi_t\rangle- \OT(p,q) \leq \underbrace{\frac{\log(mn)}{\beta_t}}_{\text{Entropic error}} + \underbrace{\frac{8\alpha_t \beta_t}{\alpha_t+\beta_t} \Big( \frac{4\Vert \log(p)\Vert_\infty}{\beta_t}  + \Vert c\Vert_\mathrm{osc}  \Big)}_{\text{Relaxation error}}
\end{align}
where $\Vert c\Vert_\mathrm{osc} \coloneqq \max_{i,j} c_{i,j} - \min_{i,j} c_{i,j}$. In particular, for $\beta_t = (1+t)^\kappa$ with $\kappa\in {]0,1[}$, it holds
\begin{align}\label{eq:rate-reg-path}
\langle c,\bar \pi_t\rangle- \langle c, \pi^*\rangle = \underbrace{O(t^{-\kappa})}_{\text{Entropic error}}+ \underbrace{O(t^{\kappa-1})}_{\text{Relaxation error}}
\end{align}
and the best upper bound on the rate $O(t^{-1/2})$ is obtained with $\kappa=1/2$.
\end{theorem}

\begin{figure}
\centering
\begin{subfigure}{0.5\linewidth}
\centering
\includegraphics[scale=0.45]{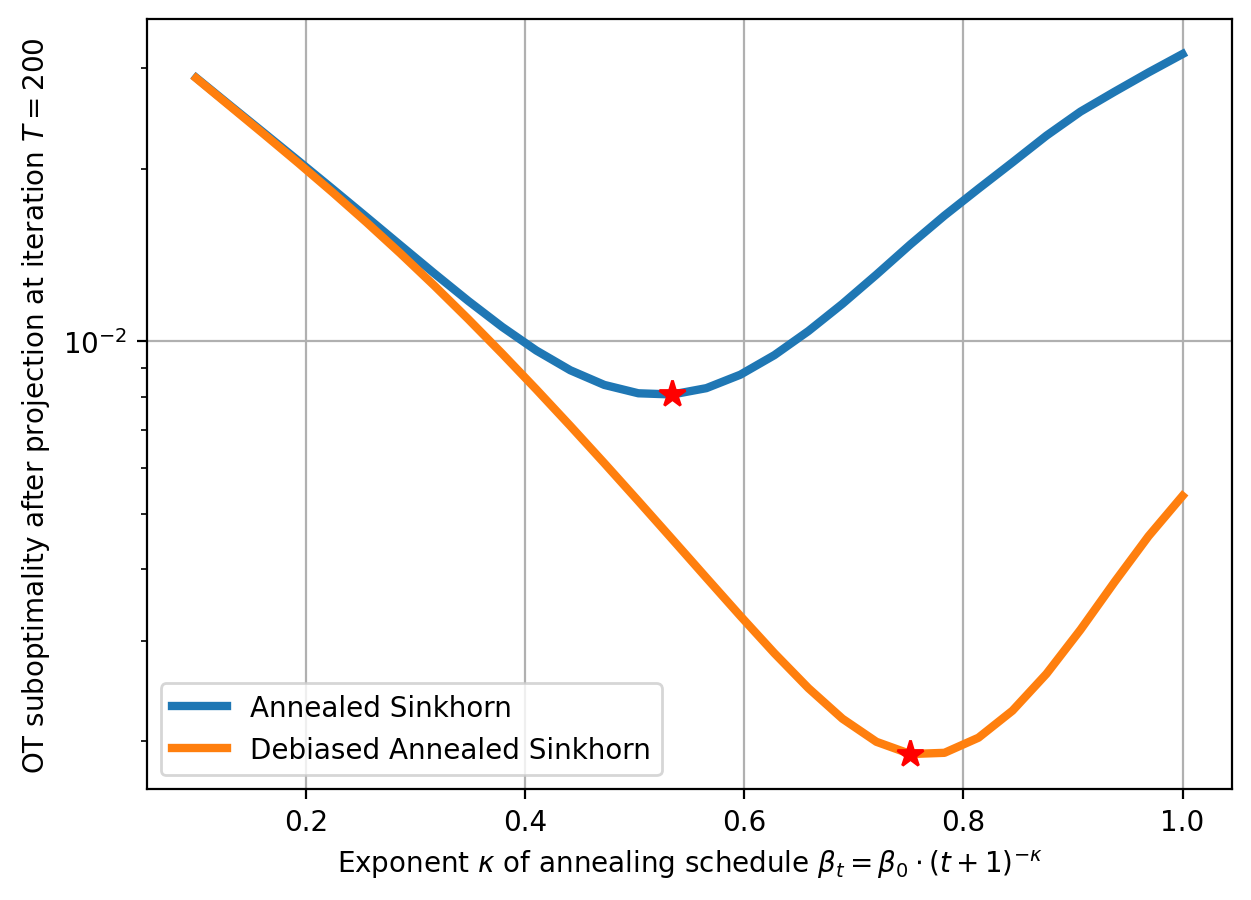}
\caption{Suboptimality at $t=200$ vs.~exponent $\kappa$}\label{fig:exponent}
\end{subfigure}%
\begin{subfigure}{0.5\linewidth}
\centering
\includegraphics[scale=0.45]{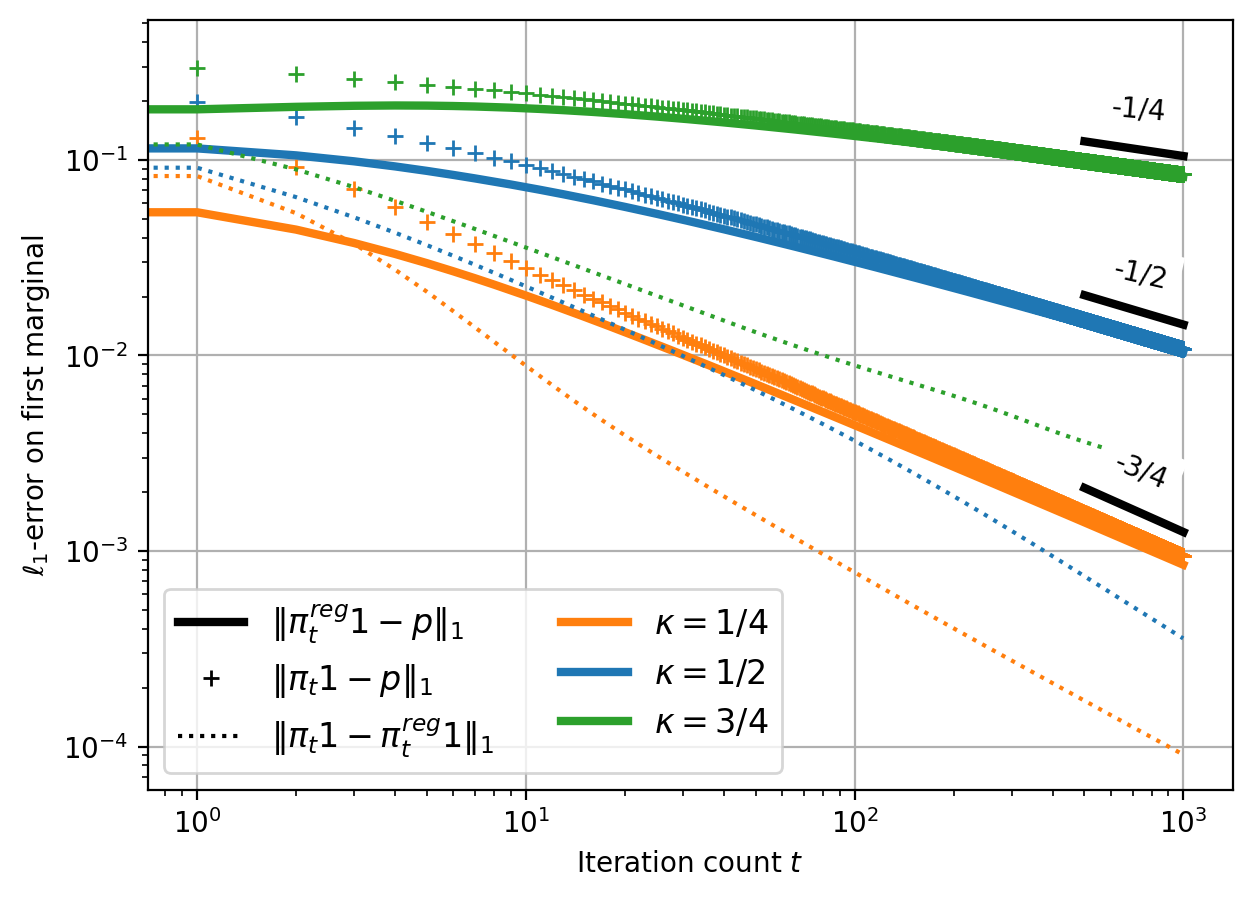}
\caption{Comparing optimization vs.~regularization paths}\label{fig:optim-vs-proxy}
\end{subfigure}
\caption{(left) Optimality gap at iteration $t=200$ as a function of the annealing exponent $\kappa$, such that $\beta_t=\beta_0 (1+t)^{-\kappa}$ (with $\beta_0=10/\Vert c\Vert_\mathrm{osc}$). The optimal exponent ({\color{red} $\star$}) are close to the predicted ones: $\kappa=1/2$ for Annealed Sinkhorn and $\kappa=2/3$ for Debiased Annealed Sinkhorn. (right) Distance between the optimization path and the regularization path (dotted lines) compared to their distances to the target. As predicted by Lem.~\ref{lem:ell1-regpath}, $\Vert \pi_t^\reg\ones -p\Vert_1$ is in $\Theta(t^{\kappa-1})$, and it closely approximates $\Vert \pi_t\ones -p\Vert_1$. The error between both paths (dotted lines) is more than an order of magnitude less. (Both experiments use the ``geometric'' cost).}
\end{figure}

As shown in Fig.~\ref{fig:exponent}, the optimal error achieved in practice is indeed with $\kappa \approx \frac12$, and the rate $O(t^{\kappa-1})$ indeed corresponds to the relaxation error (Fig.~\ref{fig:optim-vs-proxy}). Let us discuss the various terms appearing in Eq.~\eqref{eq:reg-path-bound}:
\begin{itemize}
\item \textbf{Entropic error}: The first term, in $O(\beta_t^{-1})$ is a consequence of the (implicit) entropic regularization. For the solution $\pi^*_\beta$ of $\EOT_{\beta}(p,q)$, this bound is generally of the correct order, although for very large values of $\beta$ (larger than the inverse ``discrete energy gap''), the entropic error decays exponentially~\citep{cominetti1994asymptotic, weed2018explicit}. This phenomenon is however specific to the discrete case and in continuous settings, the rate $O(\beta^{-1})$ is tight in general. For instance, it is shown in~\cite[Thm.~2.3]{malamut2023convergence} that under regularity assumptions, for the computation of the squared Wasserstein distance $W_2^2$ in dimension $d$, it holds $\langle c, \pi^*_\beta\rangle = \langle c, \pi^*_0\rangle + \frac{d}{\beta}+o(\beta^{-1})$.
\item \textbf{Relaxation error:} The second term, after using that $\alpha_t = o(\beta_t)$, is in $O(\alpha_t)$. It is a consequence of the implicit $\KL$-relaxation induced by Annealed Sinkhorn, a phenomenon exhibited by our analysis. This rate is also tight in general -- as long as the Kantorovich potentials $(u_\infty,v_\infty)$ for $\OT(p,q)$ are \emph{not} constant -- because then, combining Eq.~\eqref{eq:proxy-primal-dual} and Lem.~\ref{lem:gibbs-plan-optimal} we have asymptotically as $\beta_t\to \infty$ and $\alpha_t\to0$,  that 
\begin{align*}
\pi_t^\reg\ones = e^{-\alpha_t u^\reg_t}\odot p \propto e^{-\alpha_t u_\infty+ o(\alpha_t)}\odot p ,&& \text{hence} && \Vert \pi_t^\reg\ones - p\Vert_1 = \Theta(\alpha_t).
\end{align*}
\end{itemize}

When compared to the typical usage of Annealed Sinkhorn in practice -- where annealing schedules much faster than $\sqrt{t}$ are used -- one might interpret Thm.~\ref{thm:regularization-path} as a negative result. It tells us that no matter how nice the structure of the OT problem (such as continuous densities, multiscale structure, etc), the convergence rate of Annealed Sinkhorn cannot be faster than $1/\sqrt{t}$, achieved with the schedule $\beta_t=\Theta(\sqrt{t})$ (and only this one). Note that this does not exclude the possibility to get stronger benefits from Annealed Sinkhorn in case the goal is to solve entropic OT rather than OT (i.e. when the final inverse temperature is $\beta_\infty<+\infty$).

\begin{proof}
By Lem.~\ref{lem:altschuler} and since $\pi_t^\mathrm{reg}\in \Gamma(\ast,q)$, it holds
$$
\langle c, \bar \pi_t\rangle -\langle c, \pi^*\rangle \leq \beta_t^{-1}\log(mn)+4\Vert \pi^{\mathrm{reg}}_t\ones  - p\Vert_1.
$$
It remains to plug-in the bound from Lem.~\ref{lem:ell1-regpath} to obtain the theorem.
\end{proof}

\begin{lemma}\label{lem:ell1-regpath}
The regularization path $(\pi_t^{\mathrm{reg}})$ satisfies
$$
\Vert \pi_t^{\mathrm{reg}}\ones-p\Vert_1 \leq \frac{2\alpha_t \beta_t}{\alpha_t+\beta_t} \Big( \frac{4\Vert \log(p)\Vert_\infty}{\beta_t}  + \Vert c\Vert_\mathrm{osc}  \Big).
$$
\end{lemma}
\begin{proof}
Let us fix some iteration $t\in \NN^*$, and drop the indices $t, \alpha_t,\beta_t$ to simplify notations.  First note that a naive bound directly using Pinsker's inequality $\Vert \pi^{\reg}\ones -p \Vert_1\leq \sqrt{2\KL(\pi^{\reg}\ones|p)}$ and then bounding this last term by using the optimality of  $\pi^\reg$ would give loose estimates of the wrong order. To obtain tighter estimates, let us use the duality structure of $F^\reg$. By Lem.~\ref{lem:duality-proxy}, the convex optimization problem \eqref{eq:Fprox}, admits the dual formulation
\begin{align}\label{eq:proxy-dual-bis}
\max_{u\in \RR^n, v\in \RR^m} \frac{1}{\alpha}\langle 1-e^{-\alpha u}, p\rangle + \langle v, q\rangle -\frac{1}{\beta}\langle e^{\beta (u\ones^\top +\ones v^\top - c)},\pi^{\mathrm{ref}}\rangle+\frac{1}{\beta}
\end{align}
and a primal-dual pair of optimizers satisfies $\pi^\reg\ones = e^{-\alpha u}\odot p$ and $\pi^\reg=e^{\beta(u\ones^\top +\ones v^\top - c)}\odot \pi^{\mathrm{ref}}$. In particular, we deduce the estimate
$$
\Vert \pi^\reg\ones-p\Vert_1 = \Vert e^{-\alpha u}\odot p-p\Vert_1 \leq \min\{\Vert e^{-\alpha u}-1\Vert_\infty,2\}.
$$
It remains thus to obtain $\ell_\infty$ bounds on the optimal dual variable $u$. The first order optimality condition of $u$ in Eq.~\eqref{eq:proxy-dual-bis} reads, for an optimal dual pair $(u,v)$ and $\forall i\in \{1,\dots,n\}$,
\begin{align*}
p_i e^{-\alpha u_i} = e^{\beta u_i} \sum_{j=1}^n e^{\beta(v_j-c_{i,j})}\pi_{i,j}^\mathrm{ref}, && \Leftrightarrow && u_i = \frac{1}{\alpha+\beta}\log(p_i) -\frac{1}{\alpha+\beta} \log \Big( \sum_{j=1}^n e^{\beta(v_j-c_{i,j})}\pi_{i,j}^\mathrm{ref}\Big).
\end{align*}
It follows that for any pair of indices $i,i'\in \{1,\dots,m\}$ it holds
$$
\vert u_i-u_{i'}\vert \leq \frac{2\vert \log(p_i/p_{i'})\vert}{\alpha+\beta}  + \frac{\beta\Vert c\Vert_\mathrm{osc} }{\alpha+\beta} 
$$ 
where the last term uses the fact that the softmax function  is $1$-Lipschitz continuous and that $\max_{j} \vert c_{i,j}-c_{i',j}\vert \leq \Vert c\Vert_\mathrm{osc}$. On the other hand, we know that $p\odot e^{-\alpha u}\in \Delta^m$ so it holds
\begin{align*}
e^{-\alpha \max\{u_i\}}\leq \sum e^{-\alpha u_i}p_i=1 \leq e^{-\alpha \min\{u_i\}}&&\Rightarrow&& \min\{u_i\}\leq 0\leq \max\{u_i\}.
\end{align*}
All in all, we obtain
$$
\Vert u\Vert_\infty \leq \frac{\beta}{\alpha+\beta} \Big( \frac{4\Vert \log(p)\Vert_\infty}{\beta}  + \Vert c\Vert_\mathrm{osc}  \Big).
$$
Using the convexity of the exponential function, it is not difficult to show that $\forall u \in [-1,1]$, it holds $\vert e^{-u}-1\vert \leq 2\vert u\vert$.  Hence we conclude that 
$$
\Vert \pi^{\mathrm{prox}}\ones-p\Vert_1 \leq \frac{2\alpha \beta}{\alpha+\beta} \Big( \frac{4\Vert \log(p)\Vert_\infty}{\beta}  + \Vert c\Vert_\mathrm{osc}  \Big).
$$
This inequality holds unconditionally because, when the upper bound on $\Vert \alpha u\Vert_\infty$ is larger than $1$, i.e. when the estimate $\vert e^{-u}-1\vert \leq 2\vert u\vert$ does not apply anymore, then the right hand side is larger than $2$, and we can always trivially bound the $\ell_1$-norm by $2$.
\end{proof}

\begin{lemma}[Duality]\label{lem:duality-proxy}
Consider the optimization problem in Eq.~\eqref{eq:Fprox} defining the regularization path for some $\alpha, \beta>0$. Let $F^*$ be its minimum and $\pi^*$ its minimizer. Then it holds
\begin{align}\label{eq:proxy-dual}
F^*= \max_{u\in \RR^n, v\in \RR^m} \frac{1}{\alpha}\langle 1-e^{-\alpha u}, p\rangle + \langle v, q\rangle -\frac{1}{\beta}\langle e^{\beta(u\ones^\top +\ones v^\top - c)},\pi^{\mathrm{ref}}\rangle+\frac{1}{\beta}
\end{align}
and, denoting $(u^*,v^*)$ the unique maximizer of this dual problem (which exists), it holds
\begin{align}\label{eq:proxy-primal-dual}
\pi^*\ones = e^{-\alpha u^*}\odot p&&\text{and}&&\pi^*=e^{\beta(u^*\ones^\top +\ones (v^*)^\top - c)}\odot \pi^{\mathrm{ref}}
\end{align}
\end{lemma}
\begin{proof}
This is an application of Fenchel's duality theorem. The minimization problem of Eq.~\eqref{eq:Fprox}  is of the form
\begin{align*}
\min_{\pi\in \RR^{m\times n}} A(\pi) + B_1(\pi\ones) + B_2(\pi^\top \ones) &&\text{where}&& 
\left\{
\begin{aligned}
A(\pi) &= \frac{1}{\beta} \KL(\pi |\pi^{\mathrm{ref}}) +\langle c, \pi\rangle\\
B_1(\tilde p) &= \frac{1}{\alpha}\KL(\tilde p|p)\\
B_2(\tilde q) &= \iota_{\{q\}}(\tilde q)
\end{aligned}
\right.
.
\end{align*}
The Fenchel conjugates of $A$, $B_1$ and $B_2$ are given by
\begin{align*}
A^*(w) = \frac1\beta\langle \pi^{\mathrm{ref}}, e^{\beta( w-c)}-1\rangle, && B_1^*(u) = \frac{1}{\alpha}\langle p,e^{\alpha u}-1\rangle, && B^*_2(v) = \langle q, v\rangle.
\end{align*}
Now we have by Fenchel's duality
\begin{align*}
\min_{\pi} A(\pi) + B_1(\pi\ones) + B_2(\pi^\top \ones) &= \min_\pi \max_{u,v} \langle -u, \pi\ones\rangle -B_1^*(-u)+\langle -v, \pi^\top\ones\rangle -B_2^*(-v) +A(\pi)\\
&=\max_{u,v} -B_1^*(-u)-B_2^*(-v)-\max_{\pi} \langle u\ones^\top +\ones v^\top, \pi\rangle -A(\pi)\\
&=\max_{u,v} -B_1^*(-u)-B_2^*(-v)-A^*(u\ones^\top +\ones v^\top).
\end{align*}
This last expression coincides with Eq.~\eqref{eq:proxy-dual} after the change of variables $(u,v)\leftarrow (u/\alpha, v/\alpha)$. Here the interversion of min-max is justified by Fenchel's duality theorem applies~\cite[Thm.~31.2]{rockafellar1970convex} which applies because the dual objective is continuous. The first-order optimality conditions in the min-max problem give the primal-dual relations~\cite[Thm.~31.3]{rockafellar1970convex}. Uniqueness of the dual maximizer follows from the strict concavity of the dual (the one dimensional invariant direction discussed in App.~\ref{app:proofs} for EOT disappears here because of first term of the dual objective is strictly concave in $u$).
\end{proof}

\section{Mitigating the relaxation error}\label{sec:debiasing}
In Thm.~\ref{thm:regularization-path}, the error of the regularization path contained two terms: the entropic error in $O(\beta_t^{-1})$ and the relaxation error in $O(\alpha_t)$. Any technique to mitigate these sources of errors, can be used to improve the convergence speed of Annealed Sinkhorn. In ``geometric'' settings -- such as when the goal is to compute Wasserstein distances between continuous distributions on $\RR^d$ -- some techniques are known to reduce the entropic bias, for instance by substracting ``self-transport'' terms (of the form $\EOT_\beta(p,p)$) as in the definition of the Sinkhorn divergence~\citep{feydy2019interpolating, chizat2020faster}. In this section, we focus on the second source of error, the relaxation bias, and present a technique to mitigate it.

\subsection{Piecewise constant annealing schedule} 
Since the magnitude of the relaxation error is proportional to the change of inverse temperature $\alpha_t=\beta_{t}-\beta_{t-1}$, a first approach to mitigate it is to use a piecewise constant annealing schedule. This gives time to the relaxation error to vanish during each plateau. As observed in Fig.~\ref{fig:stopped} (where the plateau length is chosen proportional to $t$), this approach indeed leads to a reduced relaxation error at the end of each plateau. However, the sequence of peaks of error occurring at each change of temperature does not seem to improve over the rate $\Theta(t^{\kappa-1})$ suggested by our theory. Hence in the long run, this technique seems unlikely to beat the $\Theta(t^{-1/2})$ rate.

\begin{algorithm}
\caption{{\color{red}Debiased} Annealed Sinkhorn}\label{alg:debiased-annealed-sinkhorn}
\begin{enumerate}
\item \textbf{Input}: probability vectors $p\in \Delta_m^*$, $q\in \Delta_n^*$, cost $c\in \RR^{m\times n}$, annealing schedule $(\beta_t)_{t\geq 0}$
\item \textbf{Initialize}: let $b_0=\ones \in \RR^n$ and $K_0=e^{-\beta_0 c}\in \RR^{m\times n}$
\item \textbf{For $t=1,2,\dots$ let}
\begin{align*}
a_{t} &= ({\color{red} a_{t-1}^{1-\frac{\beta_{{(t-2)}\vee 0}}{\beta_{t-1}}}} \odot p) \oslash (K_{t-1} b_{t-1}), && \text{{\color{gray}$\#$ in red, the debiasing term} }\\
K_{t} &= e^{-\beta_{t} c},&&\\
b_{t} &= q\oslash (K_{t}^\top a_{t})\\
\pi_t &= \diag(a_t)K_t\diag(b_t)&& \text{{\color{gray}$\#$ define primal iterate (can be done offline)} }
\end{align*}
\end{enumerate}
\end{algorithm}
 
\subsection{Debiased Annealed Sinkhorn}
Inspired by our theoretical developments, we now develop a modification of Annealed Sinkhorn which effectively reduces the relaxation error in practice. Our reasoning relies on the following result.

\begin{proposition}[Asymptotic debiasing]\label{prop:asymptotic-debiasing}
Assume that $\OT(p,q)$ admits a unique pair of Kantorovich potentials $(u_\infty,v_\infty)$ (up to the usual translation invariance, see App.~\ref{app:proofs}), and let $(u_t)_{t\in \NN^*}$ be any sequence that converges, up to constant shifts, to $u_\infty\in \RR^m$. Consider the modified regularization path $(\tilde \pi_t^\mathrm{reg})$ where, in the objective $F^\reg_t$ (Eq.~\eqref{eq:Fprox}), we replace $p$ by
\begin{align}\label{eq:debiasing}
\tilde p_t \coloneqq Z_t^{-1} e^{\alpha_t u_t}\odot p &&\text{where} && Z_t = \Vert e^{\alpha_t u_t}\odot p\Vert_1.
\end{align}
Then it holds
\begin{align}\label{eq:debiasing-error}
\Vert \tilde \pi^\reg_t\ones-p\Vert_1 = o(\alpha).
\end{align}
Moreover, if $(\beta_t)$ is a concave annealing schedule such that $\beta_t\to \infty$ and $\beta_t-\beta_{t-1}\downarrow 0$, this statement holds with the choice $u_t = \frac{1}{\beta_t}\log(a_t)$ where $(a_t)_{t\in \NN}$ is defined by the Annealed Sinkhorn iterations (Alg.~\ref{alg:annealed-sinkhorn}).
\end{proposition}

\begin{proof}
Let $(\tilde \pi_t^\reg,\tilde u_t^\reg,\tilde v_t^\reg)$ be the primal-dual solutions of the regularization path  with first marginal $\tilde p_t$ (defined in Eq.~\eqref{eq:debiasing}) instead of $p$. Recall the primal-dual relations (Eq. \eqref{eq:proxy-primal-dual}) satisfied by the regularization path:
\begin{align}\label{eq:proxy-primal-dual}
\tilde \pi_t^\reg\ones = e^{-\alpha_t \tilde u^\reg_t}\odot \tilde p_t&&\text{and}&&\tilde \pi^\reg=e^{\beta_t(\tilde u^\reg_t\ones^\top +\ones (\tilde v^\reg_t)^\top - c)}\odot \pi^{\mathrm{ref}}.
\end{align}
In particular, $(\tilde u^\reg_t,\tilde v^\reg_t)$ are the optimal dual variables for $\EOT_{\beta_t}(\tilde \pi_t^\reg \ones,q)$.  Since $\tilde p_t$ converges to $p$, we have by Lem.~\ref{lem:gibbs-plan-optimal} that $\tilde u^\reg_t$ converges, up to shifts by constants, to  $u_\infty$. 

As a consequence, we have $\inf_{\lambda \in \RR} \Vert u_t-\tilde u^\reg_t +\lambda \Vert_{\infty} = o(1)$. Let us now consider $\hat u_t = u_t+z_t$ where $z_t\in \RR$ is a constant chosen so that $e^{-\alpha_t \hat u_t}\odot p\in \Delta_m$. With this modified sequence $(\hat u_t)$, we can strengthen the convergence as $\Vert \hat u_t-\tilde u^\reg_t \Vert_{\infty} = o(1)$ by using the fact that $p$, $\tilde p_t =e^{-\alpha_t\hat u_t}\odot p$ and $e^{-\alpha_t\tilde u_t^\reg}\odot \tilde p_t$ all belong to $\Delta_m$ for all $t\in \NN^*$. Therefore,
$$
\Vert \tilde \pi^\reg_t\ones-p\Vert_1 = \Vert e^{-\alpha_t \tilde u^\reg_t}\odot \tilde p_t - p\Vert_1 = \Vert e^{-\alpha_t \tilde u^\reg_t+\alpha_t \hat u_t}\odot p - p\Vert_1 = o(\alpha)
$$
which proves Eq.~\eqref{eq:debiasing-error}. 

Under the stated conditions on $(\beta_t)$, we know by Thm.~\ref{thm:AS-qualitative} that $\pi_t\ones \to p$. It follows, again by Lem.~\ref{lem:gibbs-plan-optimal}, that $u_t$ converges to $u_\infty$ up to constant shifts, and thus can be used in the previous argument.
\end{proof}

This result shows that it is possible to get rid of the first-order relaxation error term of the regularization path by replacing $p$ by
$$
\tilde p_t = e^{\alpha_t u_t} \odot p = (a_t)^{\frac{\alpha_t}{\beta_t}}\odot p
$$
where $(a_t)$ is the sequence produced by Alg.~\ref{alg:annealed-sinkhorn}. This computation motivates the definition of Debiased Annealed Sinkhorn in Alg.~\ref{alg:debiased-annealed-sinkhorn}. There, we do not normalize $\tilde p$ to the simplex because this only changes $a_t$ by an irrelevant scalar factor. 

In the hypothetical case where the error in Eq.~\eqref{eq:debiasing-error} is in fact $O(\alpha^2)$, then the suboptimality gap of the  \emph{debiased} regularization path, adapted from Eq.~\eqref{eq:rate-reg-path}, becomes 
$$
\langle c, \bar \pi_t \rangle -\OT(p,q) = O(t^{-\kappa}+t^{2(1-\kappa)})
$$
for polynomial annealing schedules of the form $\beta_t=\beta_0(1+t)^\kappa$, $\kappa\in {]0,1[}$. Here the fastest decay in $O(t^{-2/3})$ is obtained with $\kappa=2/3$. Interestingly, we often observe that this annealing exponent is near optimal for Debiased Annealed Sinkhorn, see Fig.~\ref{fig:exponent}.

Finally, let us insist that Prop.~\ref{prop:asymptotic-debiasing} is only an inspiration for Alg.~\ref{alg:debiased-annealed-sinkhorn} and is far from proving its convergence -- and a fortiori its convergence rate. There are several missing steps towards a complete proof: first Prop.~\ref{prop:asymptotic-debiasing} studies the error of the regularization path, not of the optimization path. Second, and perhaps more importantly, once one adjusts $\tilde p_t$ using the iterates themselves, this affects the subsequent iterates, and these recursive interactions are that are not covered by Prop.~\ref{prop:asymptotic-debiasing}.

\begin{figure}
\begin{subfigure}{0.5\linewidth}
\includegraphics[scale=0.45]{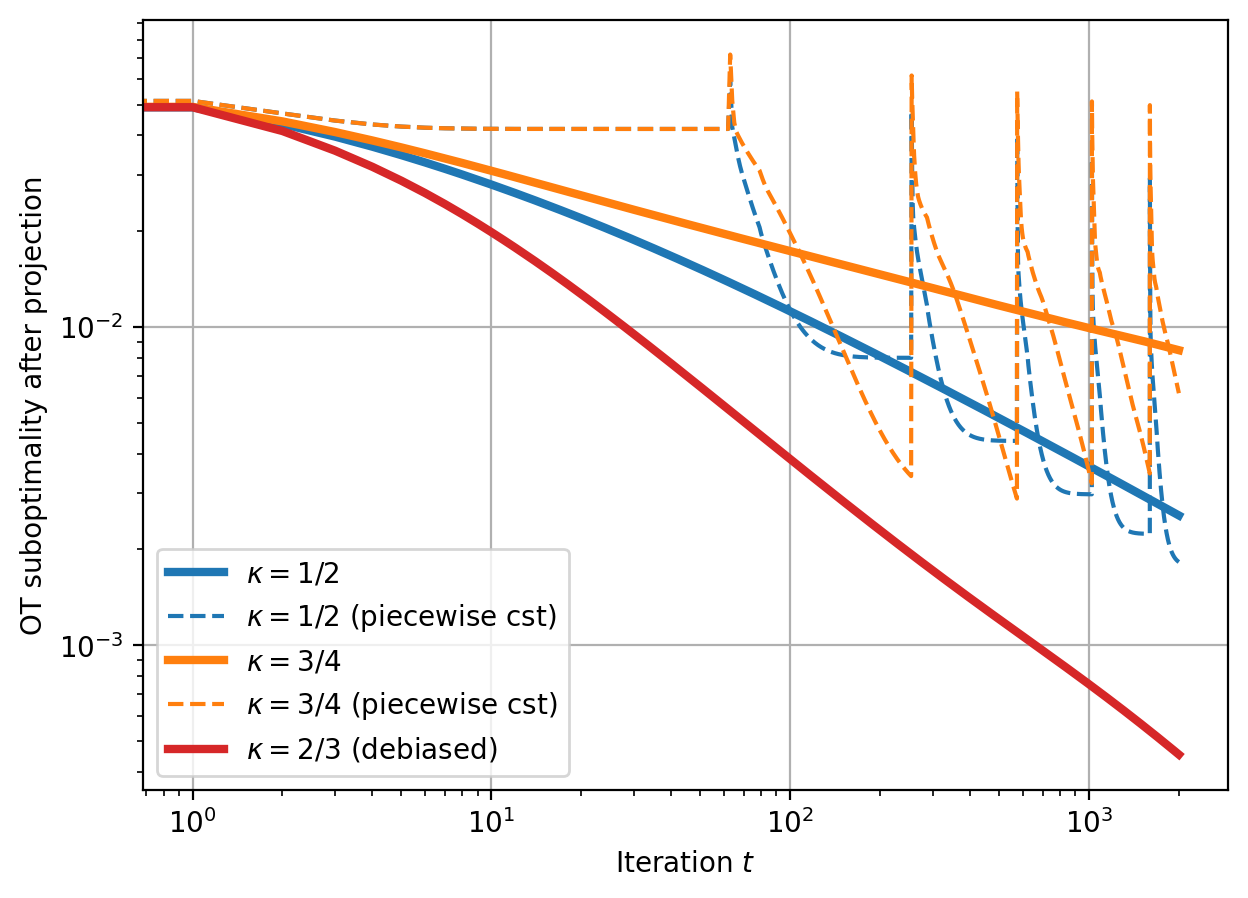}
\caption{Polynomial vs.~piecewise constant schedules}\label{fig:stopped}
\end{subfigure}%
\begin{subfigure}{0.5\linewidth}
\includegraphics[scale=0.45]{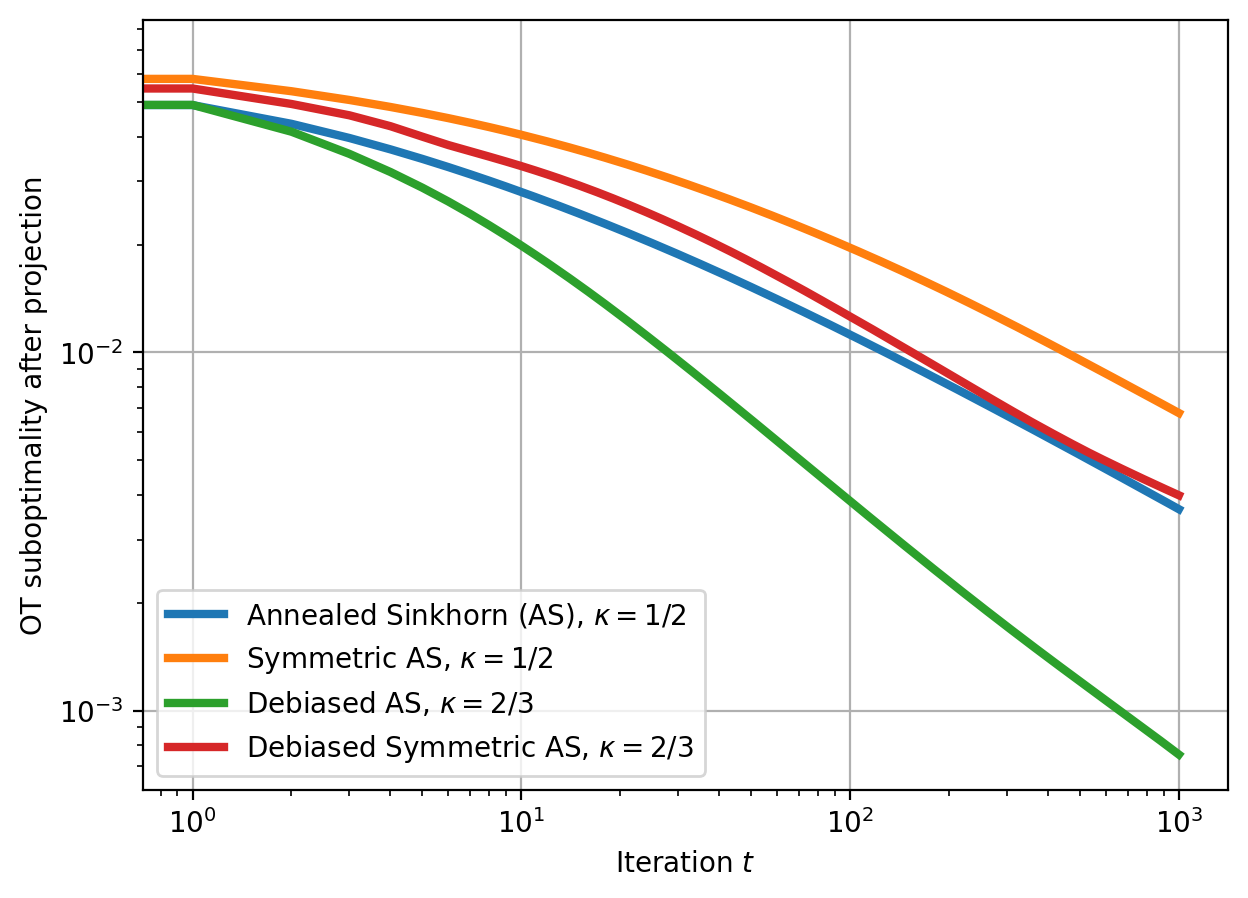}
\caption{Comparison symmetric and debiased vs not}\label{fig:symmetric}
\end{subfigure}
\caption{(left) Behavior of piecewise constant annealing schedules. We take a base schedule $\bar \beta_t =\beta_0(1+t)^\kappa$ and use an actual schedule $\beta_t$ which is updated to the value $\tilde \beta_t$ only for $t=16k^2$, $k\in \NN$ (and is constant otherwise). This standard technique leads to improvements at the end of each plateau, but appears less effective than Debiased Annealed Sinkhorn. (right) Comparison of the symmetric vs.~asymmetric versions of Annealed Sinkhorn and their debiasing. The considered problem has no symmetry, which explains why symmetric Sinkhorn underperforms. Both experiments in the ``geometric cost" setting.}
\end{figure}

\section{The case of Symmetric Sinkhorn}\label{sec:symmetric}
The fact that the output $\pi_t$ of Sinkhorn's algorithm is asymmetrical in $p$ and $q$ after a finite number of iterations -- in the sense that $\pi_t$ satisfies one marginal constraint and not the other -- can be undesirable in certain applications. This has motivated the introduction of a symmetric variant of Sinkhorn's algorithm in~\cite{knight2014symmetry}, also advocated for in~\cite{feydy2019interpolating} to solve symmetric problems of the form $\EOT_{\beta}(p,p)$.  In this section, we briefly explain how to adapt our theory and the debiasing method to this case.

The symmetric version of Annealed Sinkhorn is presented in Alg.~\ref{alg:symmetric-annealed-sinkhorn}. Let us first discuss the version without bias correction, obtained by setting the quantities in red to $0$ in Alg.~\ref{alg:symmetric-annealed-sinkhorn}. Following computations similar to those in Lem.~\ref{lem:OMD-sinkhorn}, this algorithm can be interpreted as online mirror descent (OMD) with step-size ${\color{red} \frac12}$ on the sequence of objectives
\begin{align*}
F_t(\pi) = \alpha_t\langle c, \pi\rangle +\KL(\pi\ones|p)+\KL(\pi^\top \ones |q) &&\text{where}&& \alpha_t = {\color{red} 2} (\beta_t-\beta_{t-1}).
\end{align*}
Indeed, it is not difficult to show that, defining $\pi_t = \diag(a_t)e^{-\beta_t c}\diag(b_t)$, it holds
$$
\pi_{t+1} = \argmin_{\pi} \alpha_t \langle c, \pi\rangle +\langle \log((\pi_t\ones)\oslash p)\ones ^\top, \pi\rangle +\langle \ones  \log((\pi_t\ones)\oslash p)^\top, \pi\rangle + {\color{red} 2}\KL(\pi|\pi_t) 
$$
Here the step-size for which the OMD guarantees apply is $\frac12$ because $F_t$ is $2$-relatively smooth, as the sum of \emph{two} functions which are $1$-relatively smooth (Lem.~\ref{lem:relative-smoothness}), and a linear term. As a consequence of this OMD interpretation, the theory developed in this paper applies to Alg.~\ref{alg:symmetric-annealed-sinkhorn} \emph{mutatis mutandis}. In particular:
\begin{itemize}
\item For positive, nondecreasing and concave annealing schedules $(\beta_t)$, the iterates $(\pi_t)$ of Alg.~\ref{alg:symmetric-annealed-sinkhorn} converge to a solution of $\OT(p,q)$ if and only if $\beta_t\to \infty$ and $\beta_{t}-\beta_{t-1}\to 0$;
\item The regularization path for this algorithm is now defined as the solutions of symmetric relaxed entropic OT problems:
\begin{align}\label{eq:Fprox-sym}
\min_{\pi \in \Delta_{m\times n}} \langle c, \pi\rangle + \frac{1}{\alpha_t}\KL(\pi\ones|p)+ \frac{1}{\alpha_t}\KL(\pi^\top\ones|q) +\frac{1}{\beta_t}\KL(\pi|\pi^{\mathrm{ref}}).
\end{align}
Interestingly, such problems are encountered in the theory of unbalanced OT~\citep{liero2018optimal,frogner2015learning, chizat2018scaling} and have shown to favorably compare to OT in certain applications. Thus, one potential interest of Alg.~\ref{alg:symmetric-annealed-sinkhorn} is that it approximately computes a whole regularization path of unbalanced OT problems. 
\item Finally, one can apply our debiasing method analogously, with the difference that we need now to correct both scaling factors, symmetrically. The formula in red in Alg.~\ref{alg:symmetric-annealed-sinkhorn} is obtained by replacing, at iteration $t$, $(p,q)$ by $((a_{t-1})^{\frac{\alpha_{t-1}}{\beta_{t-1}}}\odot p, (b_{t-1})^{\frac{\alpha_{t-1}}{\beta_{t-1}}})\odot q)$ in the symmetric iterates, a modification motivated in Section~\ref{sec:debiasing}. We note that in this case, it is important to project the iterate $\pi_t$ on the simplex via the normalization $\pi_t \leftarrow \pi_t/\Vert \pi_t\Vert_1$ because the total mass is not automatically preserved by the recursion.
\end{itemize}
We compare on Fig.~\ref{fig:symmetric} the performance of symmetric Annealed Sinkhorn to standard Annealed Sinkhorn on a problem without symmetry. We observe that the symmetric version is slower, which may stem from the fact that it requires a step-size $1/2$ instead of $1$. Moreover, the debiasing modification does not work as well as in the non-symmetric case, although it still brings a non-trivial improvement (we do not have a explanation for this fact). 

\begin{algorithm}
\caption{{\color{red}Debiased} Annealed Symmetric Sinkhorn (for the non-debiased version, set the exponents in red to $0$).}\label{alg:symmetric-annealed-sinkhorn}
\begin{enumerate}
\item \textbf{Input}: probability vectors $p\in \Delta_m^*$, $q\in \Delta_n^*$, cost $c\in \RR^{m\times n}$, annealing schedule $(\beta_t)_{t\geq 0}$
\item \textbf{Initialize}: let $a_0=\ones \in \RR^m$, $b_0=\ones \in \RR^n$ and $K_0=e^{-\beta_0 c}\in \RR^{m\times n}$
\item \textbf{For $t=1,2,\dots$ let}
\begin{align*}
a_{t} &= a_{t-1}^{\frac12 {\color{red} +\big(1-\frac{\beta_{(t-2)\vee 0}}{\beta_{t-1}}\big)}} \odot (p \oslash (K_{t-1} b_{t-1}))^{\frac12}, \\
b_{t} &= b_{t-1}^{\frac12 {\color{red} +\big(1-\frac{\beta_{(t-2)\vee 0}}{\beta_{t-1}}\big)}} \odot (q \oslash (K_{t-1}^\top a_{t-1}))^{\frac12},\\
K_{t} &= e^{-\beta_{t} c},&& \text{{\color{gray}$\#$ update inverse temperature} }\\
\pi_t & = \mathrm{normalize}(\diag(a_t)K_t\diag(b_t)) &&\text{{\color{gray}$\#$ project on the simplex (can be done offline)} }
\end{align*}
\end{enumerate}
\end{algorithm}

\section{Conclusion}
In this paper, we have presented the first guarantees for Annealed Sinkhorn with practical annealing schedules. Our analysis of the regularization path reveals that in its standard formulation, this algorithm converges at best at the unavoidably slow rate $O(t^{-1/2})$ to solve OT problems. This limitation can be mitigated by a debiasing strategy which enables faster annealing schedules. In numerical experiments, we observed that a single run of our Debiased Annealed Sinkhorn's algorithm spans the whole speed-accuracy Pareto front of the standard Sinkhorn's algorithm.

A widespread justification for using Annealed Sinkhorn is that it is able to take advantage of multiscale structures in geometric contexts~\cite[Chap.~3]{feydy2020geometric}. It would be interesting to develop tools to study this phenomenon, which is not captured by our approach.
Our analysis also shows that Annealed Sinkhorn and unbalanced OT are intrinsically related, a connection which will be further exploited in a companion paper.

\bibliography{LC.bib}

\appendix

\section{Experimental settings}\label{app:experiments}

The code in Julia Language~\citep{Julia-2017} to reproduce the numerical experiments can be found in this repository\footnote{\url{https://github.com/lchizat/annealed-sinkhorn}}. We consider two types of optimal transport problems:
\begin{itemize}
\item \textbf{Random cost}: we consider $m=n=100$ and a random cost matrix $c_{i,j}\overset{i.i.d.}{\sim}\Nn(0,1)$ and random weights $p_i,q_j\overset{i.i.d.}{\sim} \Uu([0,1])$ which are then normalized to unit mass.
\item \textbf{Geometric cost}: we consider two points clouds $(x_i)_{i=1}^{300}$ and $(y_j)_{j=1}^{300}$ in $\RR^2$ distributed as shown on Fig.~\ref{fig:geometric-cost}. The cost matrix is $c_{i,j}=\Vert x_i-y_j\Vert_2^2$ and $p=m^{-1}\ones_m$ and $q=n^{-1}\ones_n$.
\end{itemize}
All plots are computed with the ``geometric cost" setting, to the exception of Fig.~\ref{fig:unstructured}. In all cases, we normalize the random cost matrix to unit oscillation norm $\Vert c\Vert_{\mathrm{osc}}$. All annealing schedules are of the form $\beta_t=\beta_0(1+t)^{\kappa}$ with $\beta_0=10$. We have observed that Annealed Sinkhorn's behavior is quite sensitive to the choice of $\beta_0=10$ (in particular given the slow schedules considered here), but we refrained from changing $\beta_0$ between experiments for the sake of fair comparison.

To compute the regularization path in Fig.~\ref{fig:optim-vs-proxy}, we used the generalized Sinkhorn iterations from~\cite{chizat2018scaling}, using $100$ iterations for the first point $\pi^\reg_1$ and then $50$ iterations for the following ones, using $\pi^\reg_{t-1}$ as a warm start to compute $\pi^\reg_t$. Finally, for the piecewise constant annealing schedules shown in Fig.~\ref{fig:stopped}, we used the same form of base schedules $\bar \beta_t=\beta_0(1+t)^{\kappa}$ but where we only update $\beta_t=\bar \beta_t$ when $t=16k^2$, $k\in \NN$ which leads to plateaus of increasing lengths $16, 64, 144, 256,$ etc.

\begin{figure}\caption{Configuration of the data points used in the ``geometric setting''.}\label{fig:geometric-cost}
\centering
\includegraphics[scale=0.4]{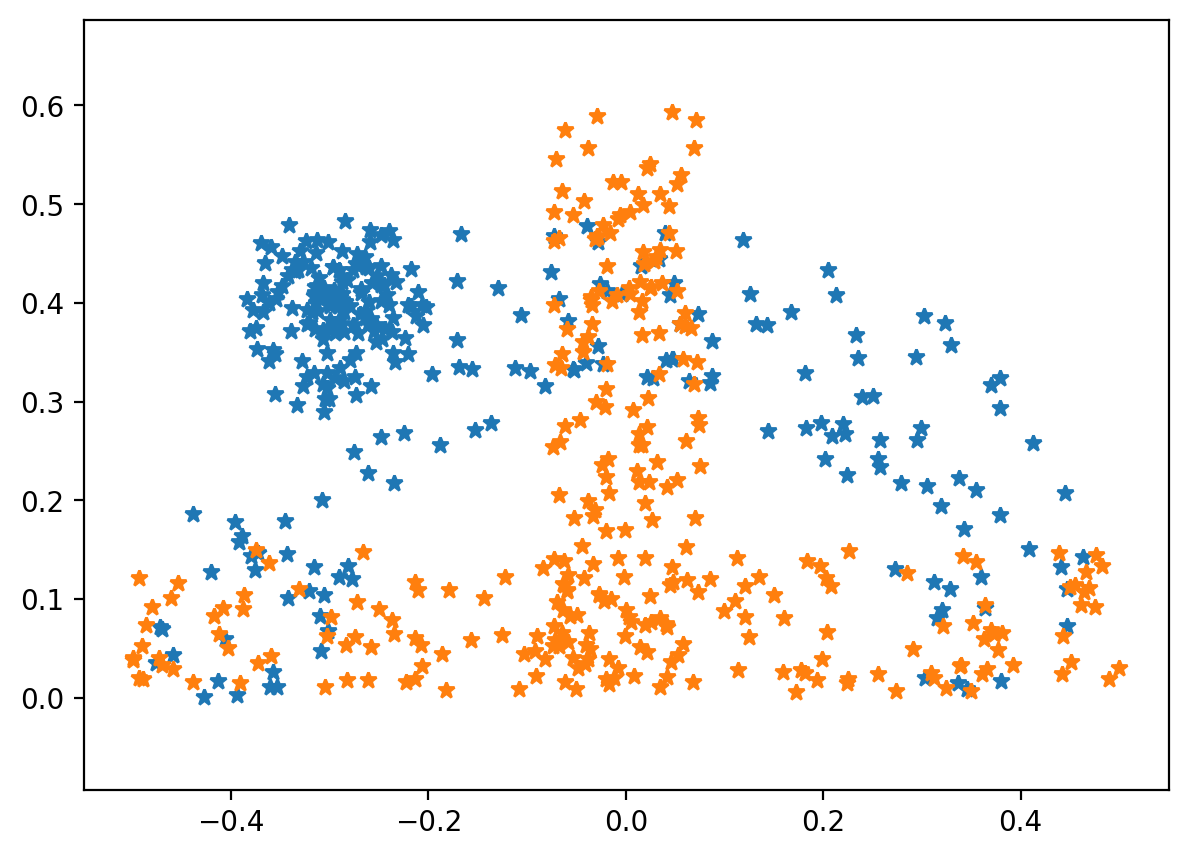}
\end{figure}

\section{Background on Entropic OT}\label{app:proofs}

In this appendix, we recall some standard facts about entropic OT. The entropic OT problem $\EOT_{\beta_t}(p_t,q_t)$ admits a dual formulation that can be derived using Lagrange multipliers as follows,
\begin{multline*}
\min_{\pi \in \Gamma(p_t,q_t)} \langle c, \pi\rangle +\beta_t^{-1} \KL(\pi|\pi^\mathrm{ref}) \\
= \min_{\pi \in \RR^{m\times n}}  \max_{u\in \RR^n, v\in \RR^m}\langle c, \pi\rangle +\beta_t^{-1} \KL(\pi|\pi^\mathrm{ref}) +\langle p_t-\pi\ones, u\rangle + \langle q_t-\pi^\top \ones, v\rangle\\
=\max_{u\in \RR^n, v\in \RR^m}  \langle p_t, u\rangle + \langle q_t , v\rangle - \beta_t^{-1} \max_{\pi \in \RR^{m\times n}}  \langle \beta_t (u\ones^\top+ \ones v^\top - c), \pi\rangle -  \KL(\pi|\pi^\mathrm{ref})
\\
=\max_{u\in \RR^n, v\in \RR^m} \langle p_t , u\rangle + \langle q_t, v\rangle +\beta_t^{-1}\Big( 1- \langle \pi^\mathrm{ref}, e^{\beta_t(u\ones^\top +\ones v^\top-c)}\rangle\Big)
\end{multline*}
where we have used the expression for the convex conjugate of $\KL$ and where the min-max interversion can be justified by Fenchel's duality theorem, see e.g.~\cite[Sec.~3.3]{merigot2021optimal}. This dual objective is constant on all $1$-dimensional affine sets of the form $\Vv(u,v) = \{(u+\lambda \ones, v-\lambda \ones)\;:\; \lambda \in \RR\}$ for any $(u,v)\in \RR^n\times \RR^m$, and is strongly concave over the orthogonal spaces $\Vv(u,v)^\perp$. The dual problem admits maximizers, which are unique up to this invariance. To account with this structure, it is convenient to consider the semi-norm $\Vert (u,v)\Vert_{o,\infty} =\inf_{\lambda\in \RR} \Vert u -\lambda\Vert_\infty+\Vert v+\lambda\Vert_\infty$. Any maximizing dual pair satisfies the first order optimality:
\begin{align}\label{eq:schrodinger-system}
\left\{
\begin{aligned}
u[i] &= \beta^{-1}_t \log(mn p_t[i]) - \beta_t^{-1}\log \sum_{j} e^{\beta_t(v[j]-c_{i,j})},\quad \forall i\in \{1,\dots,m\}\\
v[j] &= \beta^{-1}_t \log(mn q_t[j]) - \beta_t^{-1}\log \sum_{i} e^{\beta_t(u[i]-c_{i,j})},\quad \forall j\in \{1,\dots,n\}.
\end{aligned}
\right.
.
\end{align}
These equation imply in particular that $\Vert u\Vert_\mathrm{osc}\leq \Vert c\Vert_{\mathrm{osc}}+\beta_t^{-1}\Vert \log(p_t)\Vert_\mathrm{osc}$ and similarly for $v$.

\begin{lemma}[Duality and stability for EOT]\label{lem:gibbs-plan-optimal}
For any $\beta_t>0$, $a_t\in (\RR_+^*)^m$, and $b_t\in (\RR_+^*)^n$, let  $\pi_t = \diag(a_t)e^{-\beta_t c}\diag(b_t)$  and consider the marginals $p_t=\pi_t\ones\in \Delta_m^*$ and $q_t=\pi_t^\top \ones\in \Delta_n^*$. Then $\pi_t$ is the unique minimizer for $\EOT_{\beta_t}(p_t,q_t)$, and $(u_t,v_t)\coloneqq (\log(a_t/m),\log(b_t/n))$ is a maximizer of the dual EOT problem. 

Considering now sequences indexed by $t$, if we moreover have $\beta_t\to \infty$, $p_t\to p$ and $q_t\to q$ then $\EOT_{\beta_t}(p_t,q_t)\to \OT(p,q)$ and, up to taking a subsequence, $\pi_t$ converges to an optimal transport plan $\pi^*$ and $(u_t,v_t)$ converge in the norm $\Vert \cdot \Vert_{o,\infty}$ to Kantorovich potentials, i.e.~solutions to the dual of $\OT(p,q)$.
\end{lemma}
\begin{proof} 
Consider the problem $\EOT_{\beta_t}(p_t,q_t)$. Fenchel's duality theorem tells us that $(\pi_t,(u_t,v_t))$ is a primal-dual pair of solutions if and only if the optimality condition 
$$
\pi_t=e^{\beta_t(u_t\ones^\top +\ones v_t^\top -c)}\odot \pi^\mathrm{ref} = (mn)^{-1} e^{\beta_t(u_t\ones^\top +\ones v_t^\top -c)}
$$
is satisfied. By their definitions, $\pi_t=\diag(a_t)e^{-\beta_t c}\diag(b_t)$ and $(u_t,v_t)\coloneqq (\log(a_t/m),\log(b_t/n))$ satisfy this optimality condition. This shows the first part of the statement. 

Let us now prove the stability result. Let $\pi_\infty$ be a limit point of $(\pi_t)$, which must exist since the simplex is compact. We have clearly $\pi_\infty \in \Gamma(p,q)$ and by Lem.~\ref{lem:altschuler} we have for $t\in \NN^*$,
\begin{align*}
\langle c,\pi_\infty-\pi^*\rangle &=\langle c,\pi_\infty-\bar \pi_t +\bar \pi_t-\pi^*\rangle\\
&\leq \Vert c\Vert_\mathrm{osc}\cdot \Vert \pi_\infty-\bar \pi_t\Vert_1 +\beta_t^{-1}\log(mn)+4\big( \Vert p_t-p\Vert_1+\Vert q_t-q\Vert_1\big)
\end{align*}
where $\bar \pi_t$ is the output of Alg.~\ref{alg:projection} applied to $\pi_t$. Since the right-hand side converges to $0$ as $t\to \infty$, this shows that $\pi_\infty$ is an OT plan. In addition, since $\beta_t^{-1} \KL(\pi_t|\pi^{\mathrm{ref}})\leq \beta_t^{-1} \log(mn)\to 0$, we also have $\EOT_{\beta_t}(p_t,q_t)\to \OT(p,q)$. 

It remains to study the convergence of dual variables. 
Using the invariance of the dual problem, we can arbitrarily fix, say, $u_t[1]=0$ and then by Eq.~\eqref{eq:schrodinger-system} the sequence $(u_t,v_t)$ is bounded and converges, up to a subsequence, to a limit $(u_\infty, v_\infty)$. It remains to show that this is a solution to the dual of $\OT(p,q)$. Since $\pi_t=e^{\beta_t(u_t\ones^\top +\ones v_t^\top -c)}$ converges to an OT plan and $\beta_t\to \infty$, it necessarily holds $u_\infty\ones^\top +\ones v_\infty^\top -c\leq 0$. Moreover, for all $i,j$ such that $\pi_\infty[i,j]>0$, it must hold $u_\infty[i] +v_\infty[j] -c[i,j]=0$. These two properties are sufficient optimality conditions for the dual OT problem~\cite[Sec.~2.5]{peyre2019computational}, and this shows the result. 
\end{proof}

\section{Online Mirror descent}\label{sec:MD}
Let $Q$ be a convex closed subset of $\RR^d$ and $h:Q\to \RR$ a differentiable convex function -- the so-called distance-generating function -- and let 
$$
D_h(x|y)\coloneqq h(x)-h(y)-\langle \nabla h(y),x-y\rangle,\quad \forall x,y \in Q
$$
be the Bregman divergence associated to $h$. A differentiable function $f:Q\to \RR$ is said to be $L$-relatively smooth with respect to $h$ if and only if the function $Lh-f$ is convex, or equivalently if $D_f(x|y)\leq L D_h(x|y)$ for all $x,y\in Q$.

Consider a sequence of functions $f^{(t)}=g_1^{(t)}+g_2^{(t)}$ where $g_1^{(t)}, g_2^{(t)}$ are convex and lower-semicontinuous and moreover $g_1^{(t)}$ are $L$-relatively smooth wrt to $h$. 

\begin{algorithm}
\caption{Online Mirror Descent}\label{alg:onlineMD}
\begin{enumerate}
\item \textbf{Initialize}: $x_0\in Q$, step-size $\eta>0$
\item \textbf{For $t=1,2,\dots$}
$$x_t = \argmin_{x\in Q} \;g_1^{(t)}(x_{t-1})+\langle \nabla g_1^{(t)}(x_{t-1}),x-x_{t-1}\rangle + g_2^{(t)}(x)+\frac{1}{\eta} D_h(x|x_t)$$
\end{enumerate}
\end{algorithm}

This algorithm enjoys the following guarantees~\cite{}.
\begin{proposition}\label{prop:OMD}
With a step-size $\eta= \frac1L$, online mirror descent (Algorithm~\ref{alg:onlineMD}) enjoys the following guarantees. It holds for all $t\geq 1$ that $f^{(t)}(x_t)\leq f^{(t)}(x_{t-1})$ and for all $T\geq 1$ and $x\in Q$,
\begin{equation}\label{eq:onlineMD}
\frac1T \sum_{t=1}^T (f^{(t)}(x_t)- f^{(t)}(x)\big) \leq \frac{LD_h(x|x_0)}{T}.
\end{equation}
\end{proposition}
\begin{proof}
This is an immediate adaptation of~\cite[Thm~3.1]{lu2018relatively} -- which deals with the non-composite minimization case -- so we only explain how to adapt their argument. First, the case of composite objectives is discussed in their Appendix A.2. Second, their proof can also be directly adapted to the online case by noticing that the inequality (their Eq. (28))
$$
f^{(t)}(x_t)- f^{(t)}(x) \leq LD_h(x|x_{t-1}) - LD_h(x|x_t),\quad \forall x\in Q, \forall t\geq 1
$$
remains valid in the online case since it follows from the analysis of a single step. This implies in particular, by taking $x=x_{t-1}$, that $f^{(t)}(x_t)\leq f^{(t)}(x_{t-1})$. One then deduces Eq.~\eqref{eq:onlineMD} by summing this inequality for $t=1,\dots,T$ and using the fact that $D_h(\cdot|\cdot)\geq 0$. 
\end{proof}

In the main text, we apply this general guarantee to the case where $Q=\RR_+^{m\times n}$, the distance generating function is $h_e(\pi) = \sum_{i,j} \pi_{i,j}\log(\pi_{i,j})$ and the associated Bregman divergence is the Kullback-Leibler divergence $\KL(\pi|\gamma)=\sum_{i,j} \pi_{i,j}\log(\pi_{i,j}/\gamma_{i,j})-\pi_{i,j}+\gamma_{i,j}$ if $\bar \pi_{i,j}=0\Rightarrow \pi_{i,j}=0$ and $+\infty$ otherwise. In this context, we will make use of this simple but crucial property, proved e.g.~\cite[Lem.~6]{aubin2022mirror}:
\begin{lemma}\label{lem:relative-smoothness}
Let $p\in \Delta_{m}^*$. Then the function $G(\pi)= \KL(\pi\ones|p)$ is $1$-relatively smooth with respect to $h_e$ over $\Delta_{m\times n}$.
\end{lemma}

\end{document}